%% file: iclr2027_conference.tex
\documentclass{article} 
\usepackage{iclr2027_conference_preprint,times}

\usepackage[table]{xcolor}
\definecolor{darkblue}{RGB}{0,0,200}
\definecolor{mtgreen}{RGB}{50,200,20}
\usepackage{hyperref}
\usepackage{booktabs, multirow, longtable}
\usepackage{amsfonts,amsthm,amssymb,mathtools,dsfont}
\usepackage{graphicx, algorithm, algpseudocode, wrapfig}
\usepackage{soul}
\usepackage{url}
\usepackage{subcaption}

\input{math_commands.tex}

\title{MotherTree: Meta-learning on synthetic data improves decision tree training}

\author{Ziyuan Wang\\ 
Chalmers University of Technology \\
 University of Gothenburg\\
\texttt{wangzi@chalmers.se} \\
\And
Fredrik D. Johansson \\
Chalmers University of Technology \\
 University of Gothenburg\\
\texttt{fredrik.johansson@chalmers.se}
}

\iclrfinalcopy 
\begin{document}

\maketitle

\begin{abstract}
Conventional decision tree algorithms produce effective, transparent models that can be audited, communicated, and deployed independently of the training data, but require learning every new task from scratch. In contrast, tabular foundation models demonstrate that meta-learning from a synthetic prior distribution enables strong in-context prediction for previously unseen tasks, especially in small-sample regimes. However, this approach does not produce a standalone model that can be inspected in isolation. We introduce MotherTree, a tabular transformer that meta-learns decision tree induction: given a training set for a new task, it outputs a hard, axis-aligned decision tree, equivalent in form to classically trained trees, in a single forward pass. MotherTree is pre-trained on a synthetic prior using stochastic gradient descent without requiring reference trees for supervision. On established benchmarks with controlled sample size, the approach is competitive with size-matched trees from common algorithms: recursive partitioning, gradient-based tree learning, globally optimal trees, and distillation from tabular foundation models. Notably, MotherTree consistently improves over from-scratch gradient-based learning and acts as a strong initializer: task-specific tuning of the generated tree outperforms the corresponding from-scratch learner on all benchmarks and sample sizes. These results show that meta-learning can provide effective inductive biases for learning stand-alone, small decision tree classifiers.
\end{abstract}

%
%
\section{Introduction}
\label{sec:introduction}


Small decision trees are critical components of prediction and decision systems~\citep{rudin2022interpretable}. They are one of few model classes that can be directly inspected, audited, and communicated~\citep{rudin2019stop}. Still, tree models are usually trained one by one without transfer between learning tasks, failing to benefit from similar tasks being solved before. Tabular foundation models (TFM)~\citep{hollmann2023tabpfn} suggest a different possibility by demonstrating the utility of meta-learning for tabular prediction. By pre-training on numerous data sets from a synthetic prior, TFMs have proven remarkably efficient at learning new tasks from small samples~\citep{hollmann2025accurate}. Inspired, we ask: How can decision tree induction benefit from synthetic-prior meta-learning?


Current tabular foundation models like TabICL~\citep{qu2026tabiclv} and TabPFN~\citep{hollmann2025accurate} exploit in-context learning~\citep{brown2020language}, using Transformers~\citep{vaswani2017attention} to map directly from a labeled training set to test set predictions. After pre-training on tasks from a synthetic prior, in-context TFMs perform the equivalent of both training and prediction for unseen tasks in a single forward pass.  
This strategy has shown tremendous potential in benchmarks~\citep{erickson2026tabarena}, but generalization challenges remain and classical learners, including tree-based models, are still competitive in small samples and preferable in large-sample regimes~\citep{holzmuller2024better,purucker2026beyond}. Moreover, in-context learning does not produce a reusable stand-alone model for a specific task, creating obstacles for interpretation and model sharing.

Decision tree algorithms that use recursive partitioning~\citep{breiman1984classification,quinlan1993c45}, greedily applying hand-crafted splitting rules, remain the default choice for many~\citep{scikittrees}, but do not produce optimal trees and often overfit in small samples. Optimal tree algorithms~\citep{bertsimas2017optimal,aglin2020learning} perform especially well in small-model scenarios but are highly computationally intensive and difficult to scale. Gradient-based tree learning~\citep{yang2018deep,marton2024gradtree} fits all splits and leaf values jointly but has no guarantee of optimality, and has yet to demonstrate generalizable superiority over greedy methods.  All three classes of algorithms could benefit from improving either their sample efficiency or the optimization itself. First attempts have been made to \emph{meta-learn} tree fitting algorithms~\citep{zhuang2024learning}, but rely on access to gold-standard reference trees for supervision. Scalable and efficient meta-learning of decision trees remains a largely open problem. We make the following contributions toward solving it:


\begin{itemize}
\item We formulate decision tree induction as an amortized meta-learning problem and introduce MotherTree, a transformer model that generates standalone hard axis-aligned decision tree classifiers in a single forward pass. MotherTree uses a tree-structured decoder, trainable without reference trees, enabling large-scale pre-training on data sets from a diverse synthetic prior distribution, and generalization to real-world tasks with different feature sets. 
\item We evaluate MotherTree on 64 benchmark classification tasks from TabArena and OpenML CC-18, in controlled small-sample regimes, comparing against greedy, optimal, gradient-based, and distillation-based tree learners.  MotherTree consistently outperforms the corresponding from-scratch learner, GradTree~\citep{marton2024gradtree}, e.g., exceeding its AUC by 0.069 on TabArena and 0.077 on OpenML CC-18 for $B=100$ training samples (Figure~\ref{fig:auc_context_size}).
\item By fine-tuning the generated trees on task-specific data (MotherTree-Tune), we create additional performance gains that grow with the size of the training set, making the models competitive with all four baseline classes of tree learners. Synthetic experiments on learning noiseless parity functions suggest that the utility of MotherTree lies not only in improved sample-efficiency but in improved initialization for local optimization. 
\end{itemize}

%
%
\section{Related work}
\label{sec:related}

Our goal is to amortize learning of hard, axis-aligned decision trees. \emph{Hard} trees use discrete, non-differentiable rules and deterministic routing, and \emph{soft} trees use continuous, differentiable rules with fuzzy or probabilistic routing; see surveys by~\citet{li2022survey,costa2023recent}. \emph{Oblique} trees~\citep{costa2023recent,Differentiable2025_Qiangqiang,norouzi2015efficient} use linear combinations of feature values rather than axis-aligned splits formed by a single feature and threshold. 
Common learning strategies for hard trees include greedy recursive partitioning~\citep{breiman1984classification,quinlan1993c45} and global discrete optimization~\citep{bertsimas2017optimal,hu2019optimal,aglin2020learning}. 
Both hard and soft trees may be parameterized as neural networks and trained using stochastic gradient descent~\citep{tanno2019adaptive,yang2018deep,li2022survey,Differentiable2025_Qiangqiang}. 
Soft trees use fully differentiable parameterizations ~\citep{pace2022poetree,irsoy2012soft,tanno2019adaptive} but require post-hoc rounding to form hard, axis-aligned trees. 
Moreover, gradient-descent strategies are sensitive to initialization for both hard and soft trees. What's more, all three classes of algorithms above fit trees from scratch, failing to benefit from solving such tasks before. 


The goal of meta-learning is to fit a base model that can quickly learn a new task from a small sample of task-specific data~\citep{schmidhuber1992learning,ravi2017optimization,finn2017model}. MotherNet~\citep{muller2025mothernet} applied hypernetworks~\citep{ha2016hypernetworks} to meta-learning for tabular prediction, training a transformer to take a small tabular data set as input and output the parameters of a task-specific multi-layer perceptron (MLP). 
\citet{mahfoud2025learning} considered tree learning as ``amortized structure inference'', but not in the sense of meta-learning from pretraining tasks. \citet{knauer2025oh} use a large language model to output a decision tree based purely on metadata, not on tabular samples. Distillation~\citep{frosst2017distilling} from a tabular foundation model can also meta-learn trees. Our closest related work is MetaTree~\citep{zhuang2024learning}, a Transformer for generating tree classifiers, pre-trained with supervision from ground-truth reference trees. MetaTree is limited by the need to construct high-quality ground-truth trees for each pre-training task. In contrast, MotherTree learns directly from instance labels, which proved preferable in our experiments.


%
%
\section{MotherTree: Meta-learned decision tree training }
\label{sec:methods}

\begin{figure}
    \centering
    \includegraphics[width=0.95\linewidth]{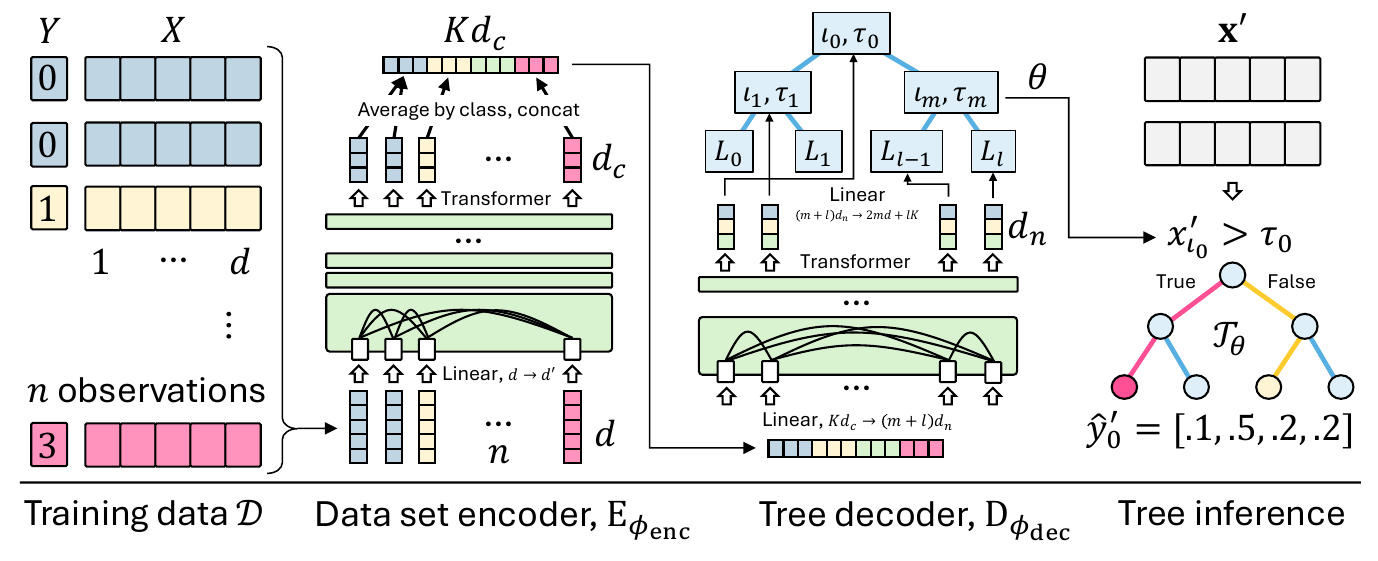}%
    \vspace{-1.0em}%
    \caption{When applying MotherTree to a new task, labeled training data is fed to the data set encoder to generate sample embeddings that are aggregated by class, forming a single data set embedding. The tree decoder maps the data set embedding to one embedding per node in the generated tree using a transformer, ultimately predicting the node split $(\iota_i, \tau_i)$ (for internal nodes) or value $L_j$ (for leaves). These define an axis-aligned, deterministic tree for stand-alone inference. \label{fig:illustration} }
\end{figure}

We propose MotherTree---a meta-learned decision tree generator. The model combines a tabular data encoder with a tree-structured decoder to learn a map $\operatorname{M}_\phi$ from labeled training set $\mathcal{D}$ to the parameters of an axis-aligned decision tree classifier, pre-trained end-to-end with instance labels as supervision. See Figure~\ref{fig:illustration}. For a new task, MotherTree uses a single forward pass to fit a tree that can be further tuned if desired. In this way, the map $\operatorname{M}_\phi$ acts exactly like a tree learning algorithm. 

\subsection{Meta-learning decision tree parameters }
\label{sec:problem_formulation}
Let $\mathcal{D}=\{(\mathbf{x}_i,y_i)\}_{i=1}^{n}$ denote a labeled training set for a classification problem (task) $\mathcal{P}$, where $\mathbf{x}_i\in\mathbb{R}^{d}$ denotes the input features and $y_i\in\{1,\ldots,K\}$ the labels of an instance $i$. Our goal is to \emph{learn} a procedure, by solving pre-training problems $\gP_1, ..., \gP_N$, that can take the unseen data set $\mathcal{D}$ and construct a hard, axis-aligned decision tree classifier, with small error on new examples from $\mathcal{P}$. We approach this problem by constructing decision trees based on a continuous parameterization.

Consider a complete binary tree with internal nodes $\mathcal{I}$ and leaf nodes $\mathcal{L}$, where $m=|\mathcal{I}|$ and $l=|\mathcal{L}|$. Following GradTree~\citep{marton2024gradtree}, we characterize the tree using parameters $\theta =(I,T,L)$, where $I,T\in\mathbb{R}^{m\times d}$ represent the feature selection indices and thresholds of internal nodes, and $L\in\mathbb{R}^{l\times K}$ the leaf values.  The topology of the tree is kept fixed as a complete binary tree with a given depth $D$. The splitting rule $s_j$ of each internal node $j\in\mathcal{I}$ is determined by 
\begin{equation}%
\label{eq:routing}%
\iota_j=\arg\max_{r\in\{1,\ldots,d\}} I_{j,r},
\qquad
\tau_j=T_{j,\iota_j},
\qquad 
s_j(\mathbf{x}) = \mathds{1}[x_{\iota_j} \leq \tau_j],%
\end{equation}%
such that $s_j(\mathbf{x})$ routes the sample $\mathbf{x}$ to $j$:s first child if true, and right otherwise. Just like standard decision trees, a tree $\mathcal{T}_{\theta}$ parameterized by $\theta = (I, T, L)$, routes each input $\mathbf{x}$ deterministically to a leaf $l(\mathbf{x})$ and outputs predicted class probabilities $\mathcal{T}_\theta(\mathbf{x}) = L_{l(\mathbf{x})} \in \Delta^{K-1}$.

Learning a tree for a problem $\mathcal{P}$ amounts to choosing parameters $\theta$ such that the classification error of the induced tree $\mathcal{T}_{\theta}$ is minimized on $\mathcal{P}$. MotherTree amortizes learning across tasks by fitting a reusable \emph{tree generator} $\operatorname{M}_\phi:\mathcal{D}\mapsto\theta$, whose parameters $\phi$ are fit on tasks from a synthetic prior. Given a training set $\mathcal{D}_{\mathrm{new}}$ for a new task $\mathcal{P}_{\mathrm{new}}$, parameters $\theta_{\mathrm{new}}$ and predictions $\hat{y}$ are computed by
\[
\theta_{\mathrm{new}}
=
(I_{\mathrm{new}},T_{\mathrm{new}},L_{\mathrm{new}})
=
\operatorname{M}_\phi(\mathcal{D}_{\mathrm{new}}),
\qquad
\hat{y}=\mathcal{T}_{\theta_{\mathrm{new}}}(\mathbf{x}),  \mbox{for an instance $\mathbf{x}$ of $\mathcal{P}_{\mathrm{new}}$}.
\]
Once parameters $\theta$ are generated, the tree $\mathcal{T}_{\theta_{\mathrm{new}}}$ can be used without access to $\operatorname{M}_\phi$ or $\mathcal{D}_{\mathrm{new}}$.

%
%
\subsection{Tabular data set encoder}
\label{sec:encoder}
The MotherTree generator $\operatorname{M}_\phi$ with parameters $\phi = (\phi_{\mathrm{enc}}, \phi_{\mathrm{dec}})$ composes a tabular data set encoder $\operatorname{E}_{\phi_{\mathrm{enc}}}$ and a tree-structured decoder $\operatorname{D}_{\phi_{\mathrm{dec}}}$. Given a labeled training set $\mathcal{D}$ of $n$ rows, the encoder first embeds $\mathcal{D}$ into row-wise representations
\(
H 
=
[\mathbf{h}_1,\ldots,\mathbf{h}_n]^\top
\in \mathbb{R}^{n\times d_c},
\)
where $d_c$ is a hyperparameter. In the default setting, $H$ is computed using the MotherNet \citep{muller2025mothernet} encoder, which applies a multi-layer transformer with row-wise attention to a linear projection of the input features (see Appendix~\ref{app:architecture} for further details). 
The row representations are then aggregated by class and concatenated into a single vector-valued data set embedding,
\[
\operatorname{E}_{\phi_{\mathrm{enc}}}({\mathcal{D}})
=
[\overline{\mathbf{h}}_1 ; \ldots ;\overline{\mathbf{h}}_K]
\in \mathbb{R}^{K d_c}
\quad
\mbox{where}
\quad 
\forall c\in [K]:
\overline{\mathbf{h}}_c
=
\frac{1}{|\mathcal{I}_c|}
\sum_{i\in\mathcal{I}_c}
\mathbf{h}_i,
\quad
\mathcal{I}_c=\{i:y_i=c\}.
\]
This embedding completes the encoder $\operatorname{E}_{\phi_{\mathrm{enc}}}$ and serves as the input to the tree-structured decoder.

\paragraph{Encoding for standalone inference.}
Like tabular foundation models, e.g., \citet{hollmann2025accurate,qu2026tabiclv}, based on prior-fitted networks (PFN)~\citep{muller2022transformers},  MotherTree is pre-trained on tasks from a synthetic prior, but differs fundamentally in how the model is used for inference. In-context PFN models jointly encode both training and test sets when making predictions, allowing test set embeddings to attend to the training set. In contrast, the MotherTree generator $\operatorname{M}_\phi$ processes only the training set, acting as a hypernetwork that directly instantiates a task-specific decision tree. Consequently, \emph{there is no test-to-training-set attention}. Once generated, the tree is a stand-alone model; test-time prediction requires only tree traversal, but no access to the training set or the tree generator.
This distinction changes the role of the data set encoder. MotherTree requires the encoder to summarize the training set for model generation rather than to measure similarities between training and test samples. Consistent with this, replacing the MotherNet encoder with the more expensive TabICLv2 encoder does not improve performance; see Appendix~\ref{app:TABICL encoder}, suggesting that a stronger in-context prediction backbone does not translate into a better tree generator.

%
%
\subsection{Ancestral decoder for axis-aligned trees}
We introduce a \textit{tree-aware transformer decoder} $\operatorname{D}_{\phi_{\mathrm{dec}}}$ to map the data set embedding to decision tree parameters, $\theta = \operatorname{D}_{\phi_{\mathrm{dec}}}(\operatorname{E}_{\phi_{\mathrm{enc}}}(\mathcal{D})).
$
For a complete binary tree with $m$ internal nodes and $l$ leaves, we construct $N=m+l$  token pre-embeddings, $\mathbf{d}^{(0)}_1, ..., \mathbf{d}^{(0)}_N$, each corresponding to a specific node in the tree, by a linear projection of the data set embedding,%
\[
H^{(0)}
\coloneqq 
[\mathbf{d}^{(0)}_1,\ldots,\mathbf{d}^{(0)}_N] 
= 
W_{\mathrm{emb}}\operatorname{E}_{\phi_{\mathrm{enc}}}(\mathcal{D})
+\mathbf{b}_{\mathrm{emb}}
\in \mathbb{R}^{N \times d_n},
\qquad
\forall i : \mathbf{d}^{(0)}_i\in\mathbb{R}^{d_n}.
\]
The node pre-embeddings are passed through a multi-layer transformer before being used to output splitting parameters $(I_j, T_j)$ for internal nodes $j$ and leaf values $L_k$ for leaves $k$. To encode the root-to-node dependency structure, we apply a masked attention with an
\textit{ancestral attention mask} $M$ that allows each node to attend only
to itself and its ancestors. Specifically, 
\begin{equation}\label{eq:attentionmask}
\mbox{for nodes $i,j \in [N]:$} \quad
M_{i,j}
=
\begin{cases}
0, & \text{if } \mathrm{Ancestor}(j,i),\\
-\infty, & \text{otherwise},
\end{cases}
\end{equation}
where $\mathrm{Ancestor}(j,i)$ indicates that node $j$ is an ancestor
of node $i$, including $i$ itself. This implicitly conditions each node on the
splits of its ancestors while preventing information flow from
unrelated branches of the tree. An $\ell$-layer transformer with mask $M$ as in \eqref{eq:attentionmask} is applied to $H^{(0)}$, producing node embeddings $H^{(\ell)} = [\mathbf{d}^{(\ell)}_1,\ldots,\mathbf{d}^{(\ell)}_N]$. The decoder $\operatorname{D}_{\phi_{\mathrm{dec}}}$ is finalized by applying a linear layer to each  embedding $\mathbf{d}_j^{(\ell)}$, producing $\theta_j=(I_j, T_j, L_j)$, and concatenating the results. 

\paragraph{Comparison with MLP decoders.} MotherNet~\citep{muller2025mothernet} uses a simple MLP decoder $\theta=\mathrm{MLP}(\operatorname{E}_{\phi_{\mathrm{enc}}}(\mathcal{D}))$.  While effective for generating the parameters of an MLP, this decoder does not explicitly model structural dependencies among tree parameters. In ablation experiments (Section~\ref{sec:ablation}), we find that ancestral decoding performs substantially better in MotherTree. Moreover, we do not utilize the low-rank strategy of the MotherNet decoder, as early ablations found no utility in this additional structure, likely partly due to the tree requiring fewer parameters than the MLP.


%
%
\subsection{Pre-training, tuning, and inference}
\label{sec:training}
The MotherTree generator $\operatorname{M}_\phi = \operatorname{D}_{\phi_{\mathrm{dec}}} \circ  \operatorname{E}_{\phi_{\mathrm{enc}}}$ is pre-trained (prior-fitted) on data sets from a large number of synthetic tasks, sampled from a structural causal model (SCM) prior $p_{\mathrm{syn}}(\mathcal{D})$. For each sampled pre-training data set $\mathcal{D}$, the generator is trained end-to-end through the tree by minimizing
\[
\mathcal{L}_{\mathrm{meta}}(\phi;\mathcal{D})
=
-\frac{1}{n}\sum_{i=1}^{n}
\log p_{\operatorname{M}_\phi(\mathcal{D})}(y_i\mid\mathbf{x}_i),
\qquad
\phi^{*}
=
\arg\min_{\phi}
\mathbb{E}_{\mathcal{D}\sim p_{\mathrm{syn}}}
\left[\mathcal{L}_{\mathrm{meta}}(\phi;\mathcal{D})\right].
\]
The estimated probability of each class is given by the corresponding element of the tree output, $\hat{p}_{\operatorname{M}_\phi(\mathcal{D})}(y_i\mid\mathbf{x}_i) = \mathcal{T}_{\operatorname{M}_\phi(\mathcal{D})}(\mathbf{x}_i)_{y_i}$. We jointly optimize the parameters of the data set encoder and tree decoder, $\phi = (\phi_{\mathrm{dec}}, \phi_{\mathrm{enc}})$, across tasks using stochastic gradient descent. Unlike MetaTree~\citep{zhuang2024learning}, this does not rely on supervision from learned reference trees. %
Computing $\mathcal{T}(\mathbf{x}_i)$, and therefore the learning objective, involves non-differentiable operations for both feature selection (arg max) and routing (threshold) as per \eqref{eq:routing}. Following \citet{marton2024gradtree}, we overcome this by passing gradients using the straight-through operator for both operations.  

\paragraph{Post-hoc parameter tuning.}
By default, we re-calibrate the generated leaf parameters $L_{\mathrm{new}}$ using task-specific $\mathcal{D}_{\mathrm{new}}$, akin to \citet{Differentiable2025_Qiangqiang}, by setting their values equal to the observed class proportions of samples reaching each leaf. Additionally, we propose a variant dubbed MotherTree-Tune that fine-tunes \emph{all} generated parameters $\theta_{\mathrm{new}} = (I_{\mathrm{new}}, T_{\mathrm{new}}, L_{\mathrm{new}})$ to $\mathcal{D}_{\mathrm{new}}$ using stochastic gradient descent, treating $\theta_{\mathrm{new}}$ as free parameters. In this way, we can use the generated tree as an \emph{initialization} for task-specific optimization using the same procedure as GradTree.

\section{Experiments}

We evaluate MotherTree in two tabular benchmarks and a synthetic task designed to illustrate the strengths and weaknesses of tree learning algorithms. We pay close attention to the comparison with GradTree as it uses the same tree parameterization but does not exploit meta-learning.
Section \ref{evaluation on benchmarks} showcases MotherTree's meta-learning capabilities in the low-sample regime, especially in binary classification, and demonstrates that MotherTree provides a strong initialization for task-specific tuning (MotherTree-Tune). Ablations demonstrate the utility of our ancestral decoder. In Section \ref{evaluation on xor}, we study the strength of MotherTree as a tree optimizer in a synthetic task with known solution.

\label{baseline and protocal}
\paragraph{Baselines.} We compare MotherTree to a suite of single-tree learning algorithms, including (i) Greedy and globally optimal trees: CART decision tree implemented in scikit-learn~\citep{scikit-learn}, Tsallis~\citep{balcan2024learning} and DL85 ~\citep{aglin2020learning}; (ii) Gradient-based trees: RADDT~\citep{Differentiable2025_Qiangqiang}, GradTree~\citep{marton2024gradtree}, DNDT~\citep{yang2018deep}; (iii) Meta tree learner:  MetaTree ~\citep{zhuang2024learning}, LLM-Tree ~\citep{knauer2025oh}; (iv) Models distilled from TabPFN2.6~\citep{hollmann2025accurate}: CART$_{\leftarrow \mathrm{TabPFN}}$ and GradTree$_{\leftarrow \mathrm{TabPFN}}$. Details regarding the distillation methods are stated in Appendix \ref{app:distillation}. We use existing PyTorch implementions where available and re-implement other gradient-based baselines in PyTorch. Hyperparameters are tuned using grid search. Details on baselines can be found in Appendix \ref{Hyperparameter tuning}. 
Since we aim to compare the learning properties of different algorithms, rather than hyperparameter tuning, we fix the topology of all trees to have the same maximum depth,  $D = 5$. 

\paragraph{Implementation Details.} MotherTree is pre-trained in two stages. In the first stage, we utilize the synthetic prior from TabICL~\citep{qu2025tabicl} with $400k$ training steps. For the second stage, the TabICLv2~\citep{qu2026tabiclv} prior is used for training $125k$ steps. The training batch size is set to 128 and 64, respectively. The maximum learning rate is set to be $3\times 10^{-5}$ with LinearLR scheduler and cosine annealing scheduler for each stage.  
Training was conducted on 4 $\times$ NVIDIA GH200 120GB. See Appendix~\ref{appendix:training-details} for details and ablations for pre-training and hyperparameters.

\paragraph{Evaluation protocol.}
We evaluate all methods under a varying training set budget $B \in \{100, 200, 400, 800, 1000\}$, where $B$ defines the \emph{maximum} number of training examples available to the learner in a given task. For each of five random seeds, we shuffle the samples and construct a train-test split subject to this budget. Details for sampling, preprocessing, and missing value handling are shown in Appendix \ref{app:sample-method} \& \ref{app:preprocess}. 
When the data set size is insufficient to allocate $B$ training examples while retaining a test set, we use a balanced train--test split instead.
Performance (uncertainty) is measured using the ROC-AUC with average (standard dev.) computed over the five splits and the data sets within each benchmark. We report the average AUC ranks in critical difference diagrams, for which staistical significance is assessed using the Friedman test followed by the Nemenyi test, following the protocol in TabArena~\citep{erickson2026tabarena}.
%
To evaluate the tree structure of gradient-based tree learners, we measure the number of uninformative paths---root-to-leaf paths on which any two splitting rules contradict each other. Finally, we report the leaf utilization---the fraction of leaf nodes which contain at least one sample, in a fixed draw of 500 test set samples.

%
%
\subsection{Evaluation on tabular benchmarks}
\label{evaluation on benchmarks}

We evaluate all methods on $64$ unique tabular classification data sets from two standard benchmarks:
(i) $30$ data sets from the small version of OpenML-CC18~\citep{bischl2017openml} 
($17$ binary, $13$ multiclass) as used in MotherNet~\citep{muller2025mothernet}, and (ii) $38$ data sets from the TabArena
benchmark~\citep{erickson2026tabarena} ($30$ binary, $8$ multiclass). The benchmarks overlap on 4  tasks; when aggregating over both, we keep one of each duplicated task. See Appendix~\ref{app:evaluation} for details.

\begin{figure}[t]
    \centering
    \includegraphics[width=0.8\linewidth]{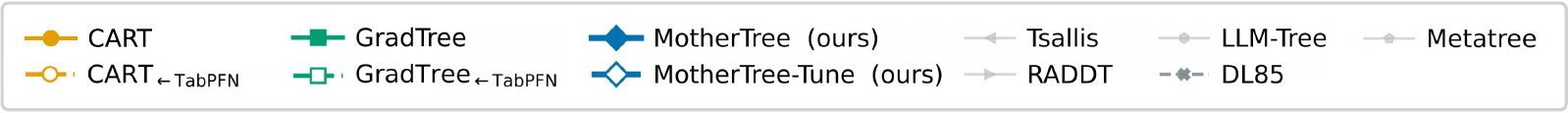}
    \begin{subfigure}[t]{0.49\linewidth}
        \centering
        \vspace{-0.5em}
        \includegraphics[width=\linewidth,trim={0 3.2cm 0 1.3cm},clip]{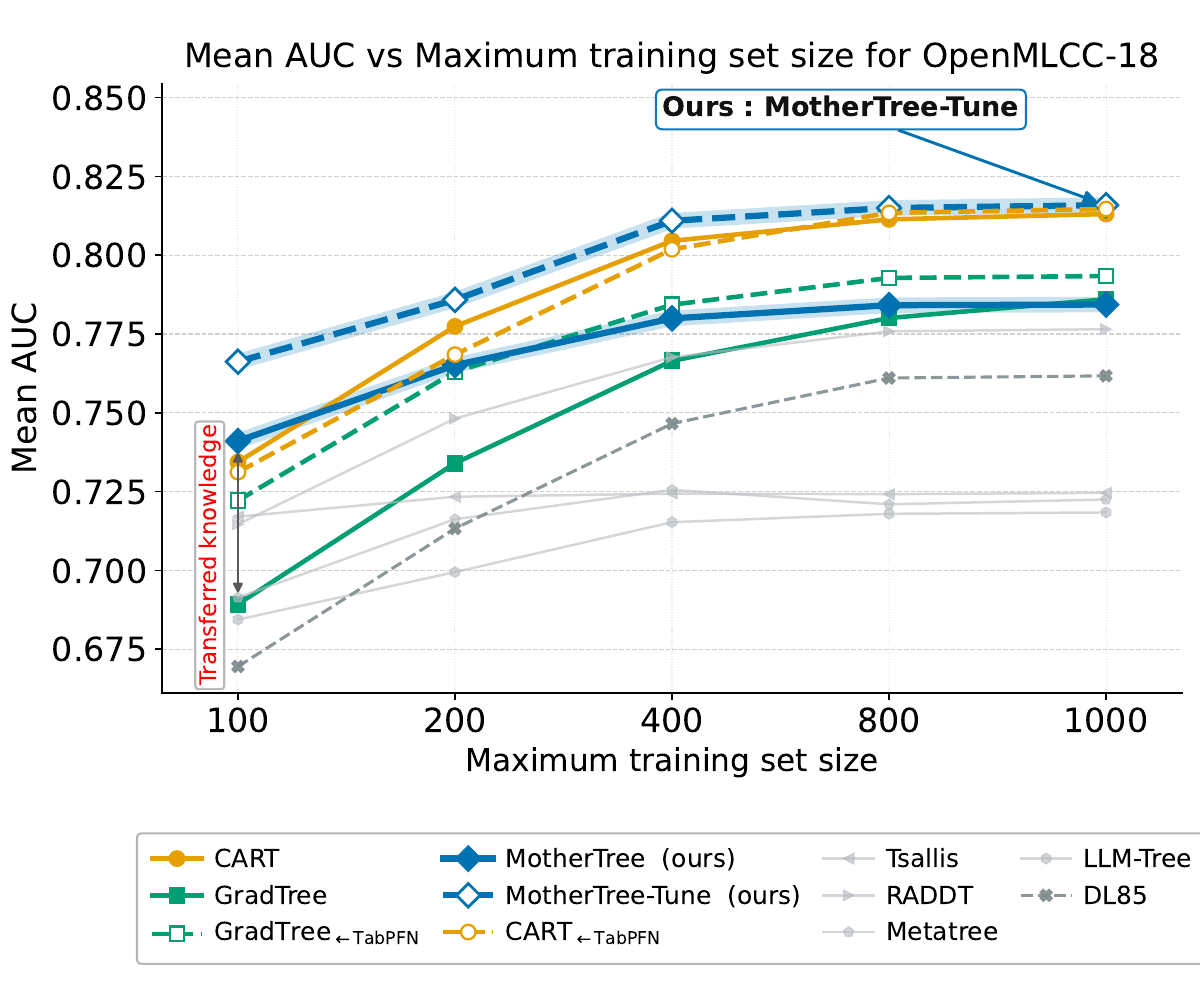}
        \caption{OpenML-CC18-Small \label{fig:auc_openml}}
    \end{subfigure}
    \hfill
    \begin{subfigure}[t]{0.49\linewidth}
        \centering
        \vspace{-0.5em}
        \includegraphics[width=\linewidth,trim={0 3.2cm 0 1.3cm},clip]{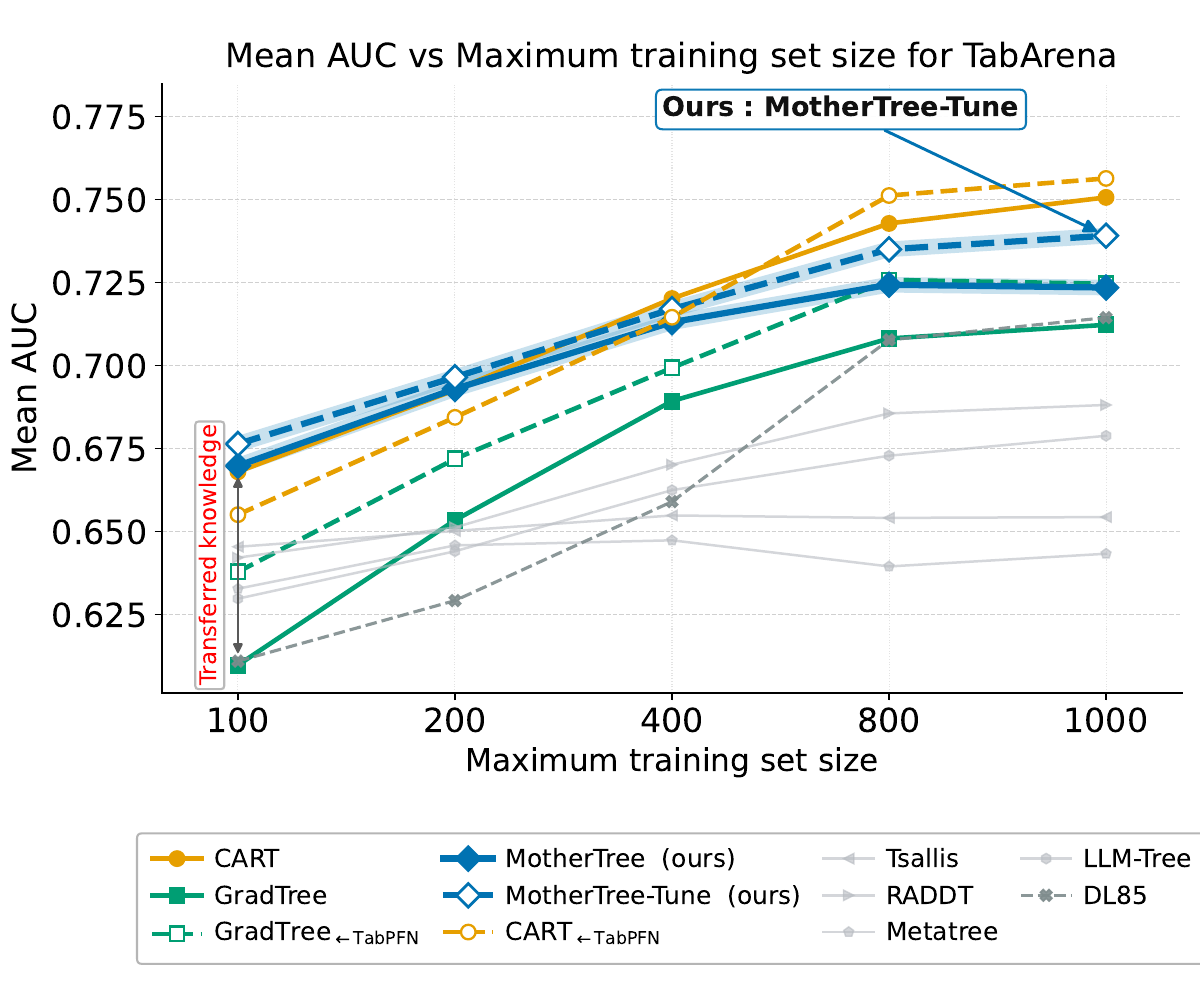}
        \caption{TabArena \label{fig:auc_tabarena}}
    \end{subfigure}
    \vspace{-0.2em}
    \caption{Mean AUC for varying training set sample budget, $B$. See also Appendix Table~\ref{tab:overall_comparison}.}
    \label{fig:auc_context_size}
    \vspace{-0.5em}
\end{figure}
\begin{table}[t]
\centering
\caption{Experimental results (avg. AUC $\pm$ std.) across tasks with training set budget $B=200$ on two benchmarks. Tables for all sample size budgets can be viewed in Appendix \ref{appendix: overall results across all lengths}. GradTree and MotherTree share tree parameterization. Best results shown in \textbf{bold}, second-best results \underline{underlined}. }
\vspace{-0.6em}
\begin{tabular}{llcccc}
\toprule
&& \multicolumn{2}{c}{TabArena {\scriptsize $(B=200)$}}
& \multicolumn{2}{c}{OpenML-CC18{\scriptsize -Small $(B=200)$}} \\
\cmidrule(lr){3-4} \cmidrule(lr){5-6}


& Method
& Binary & Multiclass
& Binary & Multiclass \\
\midrule

& CART
& $0.670_{\pm 0.107}$
& \underline{0.779}$_{\pm 0.170}$
& $0.759_{\pm 0.123}$
& \textbf{0.802}$_{\pm 0.110}$ \\

& CART$_{\leftarrow \mathrm{TabPFN}}$
& $0.659_{\pm 0.099}$
& \textbf{0.780}$_{\pm 0.165}$
& $0.748_{\pm 0.125}$
& $0.795_{\pm 0.105}$ \\

& $(\alpha^*, \beta^*)$ Tsallis
& $0.620_{\pm 0.108}$
& $0.762_{\pm 0.167}$
& $0.680_{\pm 0.155}$
& $0.780_{\pm 0.115}$ \\

& DNDT
& $0.595_{\pm 0.079}$
& $0.595_{\pm 0.055}$
& $0.650_{\pm 0.108}$
& $0.643_{\pm 0.076}$ \\

& RADDT
& $0.627_{\pm 0.100}$
& $0.744_{\pm 0.152}$
& $0.710_{\pm 0.135}$
& $0.798_{\pm 0.124}$ \\

& DL85
& $0.604_{\pm 0.088}$
& $0.722_{\pm 0.142}$
& $0.691_{\pm 0.131}$
& $0.743_{\pm 0.101}$ \\

& MetaTree
& $0.632_{\pm 0.085}$
& $0.696_{\pm 0.145}$
& $0.705_{\pm 0.141}$
& $0.731_{\pm 0.103}$ \\

& LLM-Tree
& $0.641_{\pm 0.092}$
& $0.654_{\pm 0.105}$
& $0.669_{\pm 0.145}$
& $0.739_{\pm 0.118}$ \\

\midrule
\rowcolor{mtgreen!4} 
& GradTree
& $0.635_{\pm 0.081}$
& $0.723_{\pm 0.141}$
& $0.725_{\pm 0.132}$
& $0.746_{\pm 0.118}$ \\

\rowcolor{mtgreen!4} 
& GradTree$_{\leftarrow \mathrm{TabPFN}}$
& $0.655_{\pm 0.095}$
& $0.715_{\pm 0.146}$
& $0.752_{\pm 0.128}$
& $0.774_{\pm 0.104}$ \\

\rowcolor{mtgreen!4} 
& MotherTree
& $\mathbf{0.692}_{\pm 0.104}$
& $0.698_{\pm 0.133}$
&$\underline{0.766} _{\pm 0.120}$
& $0.764_{\pm 0.103}$ \\

\rowcolor{mtgreen!4} 
\multirow{-4}{*}{\rotatebox{90}{\scriptsize \shortstack{Same tree\\param. space}}}
& MotherTree-Tune
& \underline{0.680}$_{\pm 0.101}$
& $0.757_{\pm 0.166}$
& $\mathbf{0.772}_{\pm 0.125}$
& $\mathbf{0.804}_{\pm 0.117}$ \\
\bottomrule%
\end{tabular}
\label{tab:200}
\vspace{-0.5em}
\end{table}

\paragraph{MotherTree learners are competitive across tasks and sample budgets.}
Figure \ref{fig:auc_context_size} illustrates the performance of all methods for different training set budgets. For the smallest budget, $B=100$, MotherTree and MotherTree-Tune achieve the best results on both TabArena and OpenML-CC18. This demonstrates that meta-learning can benefit tree fitting both directly (MotherTree) and as an initialization for subsequent optimization (MotherTree-Tune). CART is a consistently strong baseline (Figure~\ref{fig:ccd_context_size}), outperforming both optimal and gradient-based tree baselines at larger training set sizes. Only MotherTree-Tune remains competitive in this regime. GradTree is the strongest gradient-based baseline, and benefits substantially from distillation, but is consistently outperformed by MotherTree-Tune (we expand on this below). Appendix Table \ref{tab:runtime} shows the average fitting and inference time for evaluation tasks on both GPU and CPU. The time for MotherTree-Tune increases moderately with training set budget, remaining within a few seconds per task on both GPU and CPU.

\paragraph{Meta-learning with MotherTree outperforms from-scratch gradient-based learning.}
A close comparison with GradTree reveals the utility of meta-learning: Base MotherTree optimizes the same parameterization as GradTree using a single forward pass through the meta-learned generator, rather than from-scratch stochastic gradient descent.
%
%
From Figure \ref{fig:auc_context_size}, MotherTree outperforms GradTree up to a budget of $B=800$ samples on both benchmarks, with the methods performing comparably on OpenML-CC18 at 1000 samples. 
Table~\ref{tab:200} separates results for binary and multiclass classification tasks at $B=200$ samples.  In binary tasks, MotherTree improves the mean AUC in binary tasks over GradTree from 0.635 to 0.692 on TabArena and from 0.725 to 0.766 on OpenML-CC18-Small. MotherTree also outperform other gradient-based baselines at this budget, improving over RADDT by 0.065 (0.056) AUC and over DNDT by 0.097 (0.116) points on TabArena and (OpenML-CC18), respectively. 
For multiclass tasks, MotherTree improves over GradTree on OpenML-CC18 but underperforms on TabArena, indicating less consistent gains on this benchmark. RADDT is best among gradient-based methods at multiclass tasks with $B=200$, but \emph{returns an oblique (not axis-aligned) tree}.  The gap between MotherTree and GradTree decreases as the training data size grows, suggesting that the meta-learned prior is most beneficial when task-specific observations are scarce; fitting a tree from scratch using gradient descent becomes increasingly competitive with more data.

\paragraph{Meta-learning and task-specific optimization are complementary.}
MotherTree-Tune continues to optimize the parameters generated by MotherTree using the same procedure and hyperparameters as GradTree. Notably, MotherTree-Tune consistently outperforms GradTree and GradTree$_{\leftarrow \mathrm{TabPFN}}$ across training set budgets (Figures~\ref{fig:auc_context_size} \& \ref{fig:ccd_context_size}), benchmarks, and label types (Table~\ref{tab:200}), demonstrating the value of MotherTree as an initializer for gradient-based learning. 
Task-specific tuning is often preferable over base MotherTree: At \(B=200\), on binary TabArena tasks, MotherTree-Tune decreases AUC slightly from 0.692 to 0.680, but improves it from 0.766 to 0.772 on OpenML-CC18. The benefit is more pronounced for multiclass tasks, where tuning improves MotherTree from 0.698 to 0.757 on TabArena and from 0.764 to 0.804 on OpenML-CC18, outperforming the more flexible parameterization of RADDT. 
%
%
These results suggest that \emph{meta-learning and task-specific
optimization are complementary:} MotherTree provides a task-dependent initialization that i) consistently outperforms from-scratch optimization, and ii) can be improved further by continued tuning.

\begin{figure}[t]
    \centering
    \begin{subfigure}[t]{0.49\linewidth}
        \centering
        \includegraphics[width=.95\linewidth,trim={0.3cm 0.5cm 0.4cm 0cm},clip]{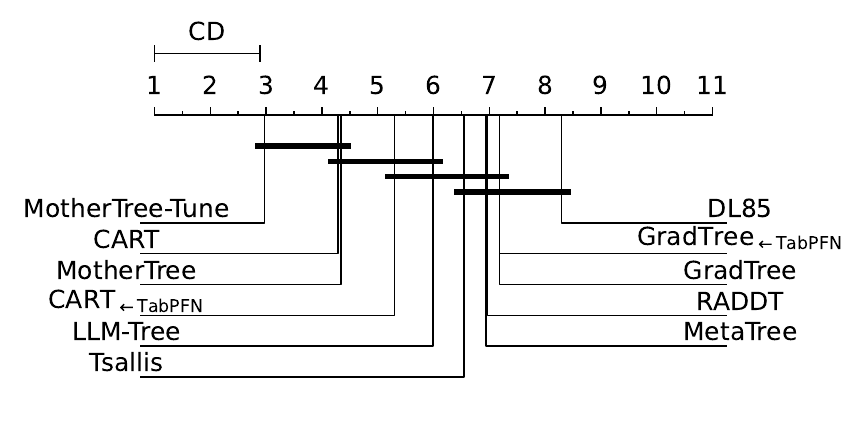} 
        \caption{Training set budget $B = 100$}
        \label{fig:joint-context-100}
    \end{subfigure}
    \hfill
    \begin{subfigure}[t]{0.49\linewidth}
        \centering
        \includegraphics[width=.95\linewidth,trim={0.4cm 0.5cm 0.4cm 0cm},clip]{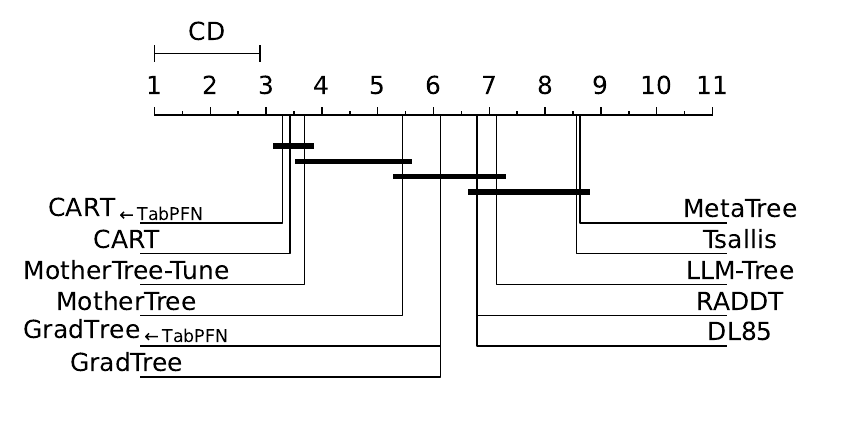} 
        \caption{Training set budget $B = 1000$}
        \label{fig:joint-context-1000}
    \end{subfigure}
    \caption{Critical Difference Diagram across both benchmarks at different training set sample budgets, $B$. More comparisons across other context budgets are reported in Appendix~\ref{app:cd}}
    \label{fig:ccd_context_size}
\end{figure}

\paragraph{Parameter meta-learning is an effective mechanism for
transferring prior knowledge to trees.}
Finally, we compare different mechanisms for transferring prior knowledge to decision trees. These methods differ substantially in how the prior is represented and transferred. Our distillation baselines transfer the predictive knowledge of TabPFN2.6 separately on each downstream task through teacher-generated labels. MetaTree meta-learns from pre-fit reference trees, thereby transferring an explicit prior over tree structures, whereas LLM-Tree constructs trees from the semantic information associated with table columns. In contrast, MotherTree meta-learns the tree-generation procedure from synthetic tasks and conditions the resulting structural prior directly on the observed context.

Figure~\ref{fig:ccd_context_size} reveals a clear dependence on the amount of
task-specific data. At a training budget of $B=100$, MotherTree-Tune
achieves the best average rank across the joint benchmark, while
MotherTree also ranks among the strongest methods. In contrast,
distillation provides limited benefit in this regime: both
CART$_{\leftarrow\mathrm{TabPFN}}$ and
GradTree$_{\leftarrow\mathrm{TabPFN}}$ rank below their meta-learned
counterparts. As the context budget increases to $B=1000$, the gap
narrows substantially, and CART$_{\leftarrow\mathrm{TabPFN}}$ becomes
the highest-ranked method. This behavior is consistent with their respective transfer mechanisms: MetaTree learns from shallow target trees, while LLM-Tree primarily relies on table semantics rather than increasing amounts of labeled context.

\paragraph{Ablation studies: Tree-aware decoder.}
\label{sec:ablation}
\begin{figure}[t]
    \centering
    \includegraphics[width=1.0\linewidth]{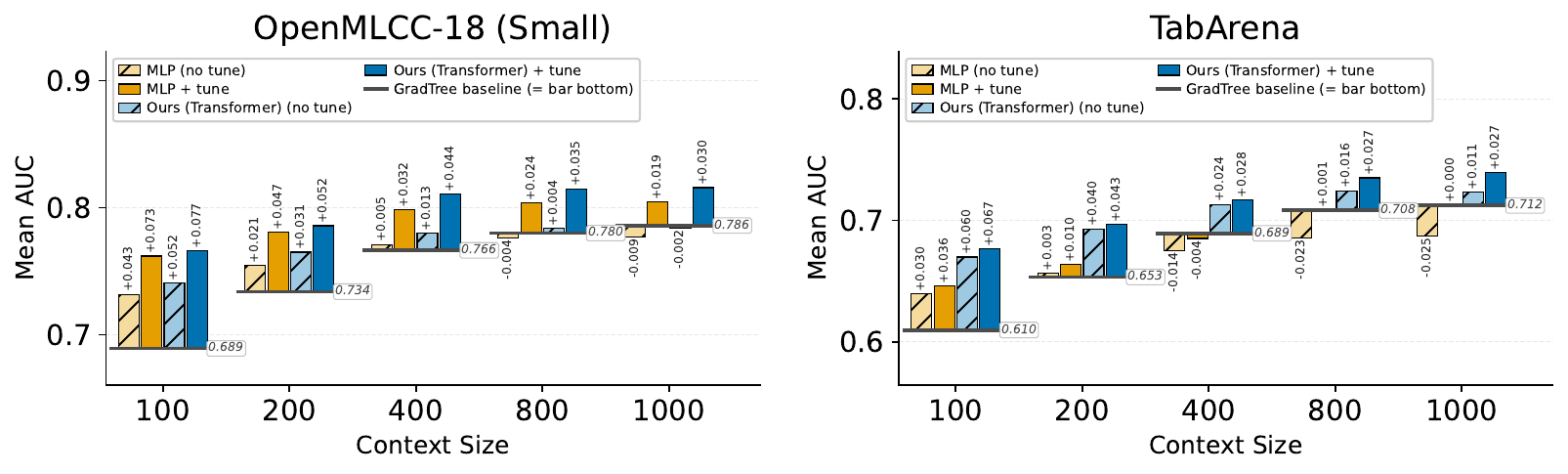}
    \vspace{-15pt}
\caption{
Decoder ablation on OpenML-CC18-Small and TabArena.
Bars show the change in mean ROC-AUC relative to GradTree at each
context size. We compare the MotherNet-style MLP decoder with our
tree-structured Transformer decoder, both before and after task-specific
finetuning. 
}
\label{fig:decoder_ablation}
\vspace{-10pt}
\end{figure}
We ablate the tree-structured Transformer decoder by replacing it with the MLP decoder used in MotherNet, modifying only its output dimension to match the parameters of the downstream tree. Figure~\ref{fig:decoder_ablation} shows that the Transformer decoder consistently provides a stronger initialization than the MLP decoder. The difference is particularly clear after task-specific tuning: Transformer-decoded trees outperform GradTree across all context sizes on both benchmarks, whereas the advantage of the MLP initialization diminishes as more training data are added. These results suggest that explicitly modeling dependencies among tree nodes provides a more effective inductive bias for generating and initializing decision trees. In Appendix~\ref{app:TABICL encoder}, we compare using the MotherNet and TabICLv2 encoders and find no benefit from the latter.

\subsection{Introspection on synthetic parity (XOR) tasks}
\label{evaluation on xor}
Real-world benchmarks do not reveal whether performance gains arise from recovering more appropriate tree structures, since the optimal splits are unknown. To test whether meta-learned tree induction improves \emph{feature selection}, we construct synthetic binary classification tasks using parity (XOR) functions corrupted with individually-informative distractors.  Every distractor is informative about the label, whereas every parity bit is informative only through its joint interaction. Thus, greedy splitting is non-optimal by design.
Given $k$ bits $\mathbf{b}$, a noiseless label is defined as $y = \bigoplus_{i=1}^{k} b_i$, where $b_i$ is drawn uniformly at random. We then form the input $\mathbf{x}$ by appending $m$ distractor features to $\mathbf{b}$, each of which is a noisy copy of the label: distractor $j$ equals $y$ with probability $q$ and is flipped with $1-q$ probability. We compare tree learners to the ``ceiling'' achievable by a distractor-only model.
For details about the Parity tasks, see Appendix \ref{XOR subsection}.

\paragraph{Higher-order parity is challenging for both gradient-based and greedy tree optimization.}
Figure~\ref{fig:xor}a first shows that the difficulty of the task is not caused solely by distractor features. For larger parity orders ($k=5$), gradient-based tree methods already exhibit a gap to the DL85 oracle and greedy splitting (CART) when no distractors are present ($m=0$). This reflects the difficulty of \emph{optimizing} the differentiable tree parameterization to fit higher-order interactions, rather than an intrinsic limitation of the tree structure itself or sample inefficiency (since the label is noiseless). We therefore focus on the additional degradation caused by distractors, rather than interpreting the absolute gap to the oracle as feature-selection error. 

\paragraph{Meta-learned initialization improves robustness under a tight depth budget.}
The effect of distractors becomes pronounced once
$\lambda \coloneqq m-(D-k)>0$. In this regime, the depth budget is exhausted, so
using a distractor consumes a split that is required to fully represent the
parity function. Most tree learners deteriorate
sharply as $\lambda$ increases. In contrast, MotherTree-Tune degrades
more gradually and remains substantially more robust across the
difficult configurations, indicating that the meta-learned
initialization helps task-specific optimization resist locally
predictive but structurally sub-optimal features. Figure~\ref{fig:xor_q_facet} isolates this effect across different distractor
strengths $q$ by normalizing performance relative to the
distractor-only ``ceiling'' and the DL85 oracle. Across values of $q$,
MotherTree-Tune retains a larger fraction of the attainable improvement
once the depth constraint requires distractors to be rejected. These
results suggest that the main benefit of meta-learning in this setting
is not to solve the parity problem outright, but to provide a stronger
initialization that makes gradient-based tree optimization more robust
to misleading shortcut features.

\begin{figure}[t]
    \centering
    \hfill
    \begin{minipage}[t]{0.49\linewidth}
        \centering
        \includegraphics[width=1\linewidth,trim={0.25cm 0.2cm 0.3cm 0.25cm},clip]
        {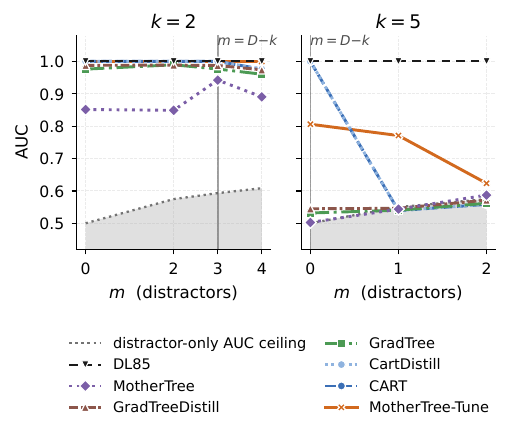}
        \subcaption{AUC against the number of distractor features $m$ for varying parity orders $k$ (number of bits). }
        \label{fig:xor_panels}
    \end{minipage}
    \hfill
    \begin{minipage}[t]{0.49\linewidth}
        \centering
        \includegraphics[width=1\linewidth,trim={0.23cm 0.2cm 0.25cm 0.25cm},clip]
        {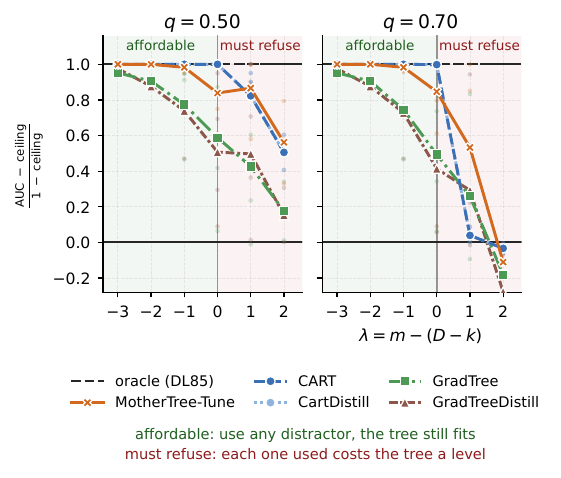}
        \subcaption{Realized improvement vs optimal distractor-only model under different distractor-label correlation, $q$.}
        \label{fig:xor_q_facet}
    \end{minipage}
    \caption{%
        Learning parity with distractors at fixed maximum tree depth, $D = 5$. Right:
        $\lambda = m - (D-k)$ is the number of distractors that must be ignored to learn an optimal tree. Green ($\lambda \le 0$): the depth budget has slack and splitting on a distractor does not rule out an optimal model. 
Red ($\lambda > 0$): the budget is exhausted, so each distractor used takes a
level away from the parity tree; at least $\lambda$ of them must be refused. Results across all settings could be seen in Appendix \ref{app:xor_full_fig}. }
\label{fig:xor}
\end{figure}

\section{Discussion}
\begin{wrapfigure}[15]{r}{0.5\textwidth}
    \centering
    \vspace{-12pt}
    \includegraphics[
        width=\linewidth,trim={0 0 0 0.2cm},clip
    ]{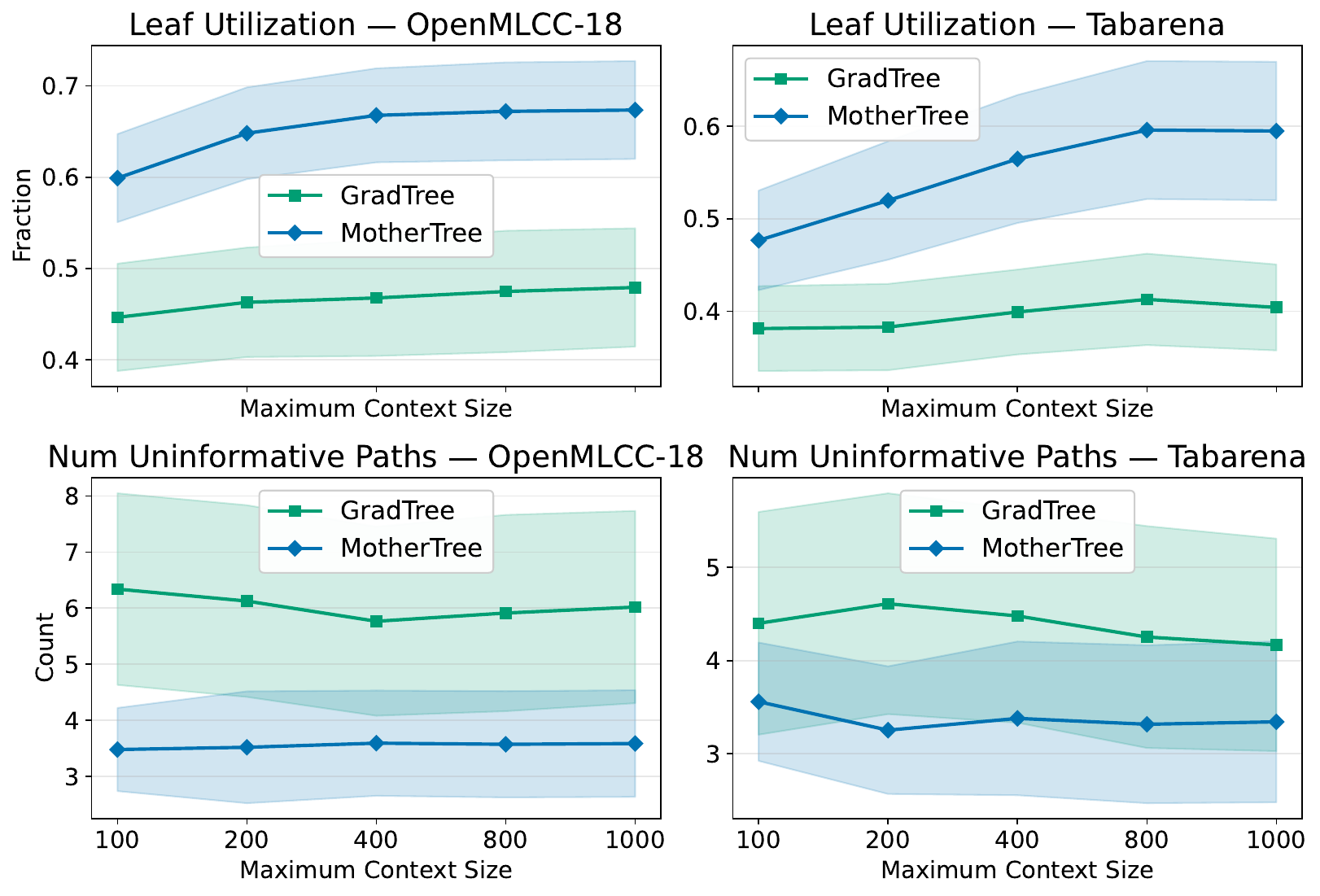}
    \captionsetup{skip=3pt}
    \caption{Number of uninformative paths (lower is better) and leaf utilization (higher is better). }
    \label{fig:tree structure}
    \rule{0pt}{2pt}
\end{wrapfigure} 
\paragraph{When does amortized tree learning help?}
In our experiments, MotherTree provides the largest gains when task-specific training sets are
limited in size; the advantage over fitting the same tree
parameterization from scratch decreases as more data become
available. This is consistent with the role of meta-learning as a
task-conditioned prior: it provides a stronger inductive bias when
observations are scarce, whereas task-specific optimization becomes
increasingly competitive with more data. Currently, the gains are strongest on
binary classification tasks, while the weaker multiclass performance suggests that direct
generation also inherits limitations of the underlying differentiable
tree parameterization. Importantly, the structural statistics from Figure \ref{fig:tree structure} show
that this prior affects the tree itself: MotherTree uses a larger
fraction of leaves and produces fewer contradictory root-to-leaf paths
than GradTree under the same depth and parameterization, which shows that the improvement from meta-learning also leads to a better tree structure. 

\paragraph{Limitations.} MotherTree currently generates complete binary trees with a fixed
topology and depth. This simplifies amortized generation, since the
number and correspondence of tree parameters are known in advance, but
prevents the model from adapting tree size or topology to individual
tasks. As a result, some generated trees may allocate capacity to
unnecessary nodes, while more complex tasks may require a larger tree
than the predefined structure allows. Extending MotherTree to generate
variable-size or pruned tree structures is an important direction for
future work.

Our study demonstrates the value of meta-learning induction of \emph{single-tree} classifiers. However, tree \emph{ensembles} universally achieve stronger predictive performance by combining predictions of multiple complementary trees. Future research extending amortized tree generation to ensembles must tackle questions about how to encourage diversity and coordinate dependency structures across generated trees, to minimize redundancy and improve coverage of the target task. 

\subsubsection*{Acknowledgments}
We thank Uri Shalit, Jesse Krijthe, and the Healthy AI Lab for fruitful discussions in the early stages of this work, and Valter Sch\"{u}tz for valuable feedback on the manuscript.
The work was supported in part by the Wallenberg AI, Autonomous Systems and Software Program
(WASP) funded by the Knut and Alice Wallenberg Foundation. The computations and data handling were enabled by resources provided by the National Academic Infrastructure for Supercomputing in Sweden (NAISS), partially funded by the Swedish Research Council through grant agreement no. 2022-06725.

\bibliography{iclr2027_conference}
\bibliographystyle{iclr2027_conference}

\clearpage
\appendix

\section{Appendix}

\subsection{Architecture details}
\label{app:architecture}
\paragraph{Encoder} We keep the same encoder from MotherNet ~\citep{muller2025mothernet}, where each row of the table is embedded through a linear layer and one-hot encoding for class is used before sending to the transformer blocks. 
Unlike TabPFN~\citep{hollmann2025accurate} and TabICL~\citep{qu2026tabiclv},
which rely on dual attention between context and features, MotherNet encoder operates exclusively on the context data and
does not require test points to participate in attention.

Concretely, given a context set with $n$ samples and $d$ features,
the encoder first applies a linear projection to map each row and corresponding label
into a $d'$-dimensional feature embedding. Addition is used to combine feature and label information, followed by $12$ 
inter-sample attention layers to capture the mutual information between rows.
This yields a row-wise representation for every context sample.
To aggregate the $n$ row representations into a single data set embedding,
the encoder performs class-wise mean pooling:
it averages the row representations belonging to the same class,
and concatenates the resulting $k$ class means into
$\operatorname{E}_{\phi_{\mathrm{enc}}}({\mathcal{D}})$.
The resulting $\operatorname{E}_{\phi_{\mathrm{enc}}}({\mathcal{D}})$ is fixed-dimensional, which the tree decoder subsequently
maps to the parameters of an axis-aligned decision tree.

\paragraph{Decoder}
For node representations
$H^{(i)}\in\mathbb{R}^{N\times d_n}$,  ancestral self-attention is
computed as
\[
\mathbf{Q}=H^{(i)}\mathbf{W}_Q,\quad
\mathbf{K}= H^{(i)}\mathbf{W}_K,\quad
\mathbf{V}=H^{(i)}\mathbf{W}_V,
\]
\[
\mathrm{Attn}(\mathbf{Q},\mathbf{K},\mathbf{V})
=
\mathrm{softmax}
\left(
\frac{\mathbf{Q}\mathbf{K}^{\top}}{\sqrt{h}}+\mathbf{M}
\right)\mathbf{V}.
\]
After $L$ Transformer layers, the resulting node representations are
projected to the corresponding split and leaf parameters, yielding
$\theta=(I,T,L)$. We set $d_n=32$ and use $L=8$ Transformer layers,
resulting in a decoder with $23.6$M parameters. These hyperparameters
were chosen to keep the decoder capacity comparable to that of the
original MotherNet decoder ($17.7$M parameters) for tree generation. 

\subsection{Comparison with TabICL encoder} 
\label{app:TABICL encoder}

\begin{wraptable}{r}{0.56\textwidth}
  \centering
\caption{Mean AUC for different encoders on TabArena.}
  \label{tab:backbone-ablation}

  \footnotesize
  \setlength{\tabcolsep}{3.5pt}
  \begin{tabular}{@{}ccc@{}}
    \toprule
    Context & TabICL Encoder & MotherNet Encoder \\
    \midrule
     100 & 0.632 & 0.639 \\
     200 & 0.652 & 0.657 \\
     400 & 0.671 & 0.675 \\
     800 & 0.679 & 0.685 \\
    1000 & 0.681 & 0.687 \\
    \bottomrule
  \end{tabular}
\end{wraptable}

We also ablate the MotherTree encoder by replacing the MotherNet encoder
with the encoder from TabICLv2~\citep{qu2026tabiclv}. We provide only the
training samples to the encoder and use the same aggregation procedure to
construct the data set embedding, as described in Section~\ref{sec:encoder}.

As shown in Table~\ref{tab:backbone-ablation}, replacing the MotherNet
encoder with TabICLv2 does not improve performance across the evaluated
context sizes. In fact, the MotherNet encoder performs slightly better in
all settings. These results suggest that the performance of MotherTree is
not driven by the use of a stronger tabular prediction backbone, and that
increased encoder sophistication does not necessarily translate into better
representations for tree generation.

\subsection{Training Details}
\label{appendix:training-details}

MotherTree is pre-trained on both TabICL~\citep{qu2025tabicl} and TabICLv2 prior~\citep{qu2026tabiclv}. The main reason is that, for TabICL prior, it provides diverse activation functions and tree-based SCMs to model hierarchical interactions between variables, which helps the model to grasp the basic complicated interaction between variables. However, due to the instability of some activation functions like \texttt{Exponential function}, it sometimes introduces the overflow error, which makes the data set quality less effective. While in TabICLv2, a more sophisticated prior generation process is utilized.  \texttt{Random Cauchy graph} creates richer connectivity of graph, but the generation process is much more expensive than the previous one. Hence, during the first stage of training,  our model is trained on the fly with generation of 128 batches of data for each step. A LinearLR scheduler is used to gradually increase the learning rate. For the second stage of training, we generate $120k$ files on the disk. For each file, it contains $4$ groups of graph configuration. For each configuration, we randomly generate 16 batches.  1000 steps of warm up is firstly applied and then we switch to cosine annealing scheduler by setting minimum learning rate to $1\times10^{-8}$.  We optimize the MotherTree using the AdamW optimizer~\citep{loshchilov2017decoupled} and the momentum coefficients are set to \(\beta_1=0.9 \) and \(\beta_2=0.999\). 

\begin{wraptable}{r}{0.36\textwidth}
  \centering
  \caption{Impact of Stage 1 pre-training($400k$ steps) on validation sets from OpenML-CC18~\citep{muller2025mothernet} after $125k$ steps.}
  \label{tab:prior-ablation}

  \footnotesize
  \setlength{\tabcolsep}{3.5pt}
\begin{tabular}{@{}lc@{}}
  \toprule
  Training Scheme & Mean AUC \\
  \midrule
  Stage 2 only & 0.685 \\
  Stage 1 + Stage 2 & 0.752 \\
  \bottomrule
\end{tabular}
\end{wraptable}

As shown in Table~\ref{tab:prior-ablation}, incorporating Stage 1 pre-training with the TabICL prior substantially improves the validation AUC compared with training only on the TabICLv2 prior. This design is also considerably more training-efficient, since generating TabICLv2 prior samples is substantially more expensive and would otherwise lead to a much longer overall training time. For training hours, Stage 1 took 68 hours and stage 2 took 18 hours.

\subsection{Evaluation Details}
\label{app:evaluation}

\subsubsection{Duplicate Tasks}
Between TabArena and OpenMLCC-18, there are four overlapping data sets, which are \texttt{blood-transfusion-service-center}, \texttt{diabetes}, \texttt{credit-g} and \texttt{qsar-biodeg}. For the first two data sets, they are exactly the same data sets but orders are different. For \texttt{credit-g}, the original ordinal encoding is reversed, which doesn't affect the results. \texttt{qsar-biodeg}
is the one case where the two copies are not the same size(TabArena has one less row) and TabArena changes some of numeric-ordinal features to categorical features.

\subsubsection{Distillation Baseline}
\label{app:distillation}
Following MotherNet~\citep{muller2025mothernet}, we also incorporate distillation baselines for comparison. Distillation provides a natural way to transfer knowledge from a powerful teacher model (e.g., TabPFN) to a lightweight student model. However, unlike the original distillation framework~\citep{hinton2015distilling}, which relies on gradient descent with soft labels, the non-differentiability of CART prevents direct application of this approach. We therefore adopt a simplified strategy: hard labels sampled from the predicted probability of the teacher model are mixed with the original training data, and the student model is trained on this augmented data set.

For consistency and simplicity, we apply this strategy to both CART and GradTree. Since both MotherTree and these distilled baselines leverage knowledge from PFN(prior fitted network), any performance gap between them can be attributed primarily to architectural differences rather than the source of knowledge. We sample the label for each training data according to the predictions from the teacher model. Then, append the sample with new label to the training set, which doubles the original size of the training set. 

As for the teacher model, we experiment with both \texttt{TabPFN2.6} and \texttt{TabPFN-3}~\citep{grinsztajn2026tabpfn}. We report results from \texttt{TabPFN2.6} in the main paper, as we observe no significant performance difference between the two teachers, while \texttt{TabPFN2.6} is substantially more time-efficient.

\subsubsection{Baselines \& Hyperparameter tuning}
\label{Hyperparameter tuning}
We compare our method against several baseline models using their publicly available implementations. 
The corresponding repositories are listed below:
\begin{itemize}
    \item DNDT: \url{https://github.com/wOOL/DNDT}
    \item RADDT: \url{https://github.com/YankaiGroup/RADDT}
    \item LLM-Tree: \url{https://github.com/ml-lab-htw/llm-trees}
    \item GradTree (we re-implemented the following repository in \texttt{Pytorch}): \url{https://github.com/s-marton/GradTree}
    \item DL85: \url{https://github.com/aia-uclouvain/pydl8.5}
    \item MetaTree: \url{https://github.com/EvanZhuang/MetaTree}
\end{itemize}

\begin{longtable}{@{} l p{11cm} @{}}
\caption{Hyperparameter search spaces and fixed settings for each model.}
\label{tab:hyperparams} \\

\toprule
\textbf{Model} & \textbf{Hyperparameters} \\
\midrule
\endfirsthead

\multicolumn{2}{c}%
{{\tablename\ \thetable{} -- continued from previous page}} \\
\toprule
\textbf{Model} & \textbf{Hyperparameters} \\
\midrule
\endhead

\midrule
\multicolumn{2}{r}{{Continued on next page}} \\
\endfoot

\bottomrule
\endlastfoot

CART &
criterion $\in\{$gini, entropy$\}$;
min\_samples\_split $\in\{2,5,10\}$;
min\_samples\_leaf $\in\{1,2,5\}$.
\\[4pt]

GradTree &
feature selector LR $\in\{0.001,0.005,0.010,0.050\}$;
threshold LR $\in\{0.010,0.050,0.100,0.250\}$;
leaf node LR $\in\{0.001,0.005,0.010,0.100\}$;
epochs $=1000$;
patience $=50$.
\\[4pt]

DL85 &
time limit $=600$s;
max depth $=5$;
quantile binning $=5$.
\\[4pt]

$(\alpha^*,\beta^*)$ Tsallis &
$a \in \mathrm{linspace}(1,10,10)$;
$b \in \mathrm{linspace}(1,10,10)$.
\\[4pt]

DNDT &
learning rate $=0.05$;
batch size $=256$;
epochs $=300$;
patience $=20$;
max features $=10$;
Gumbel temperature $1.0\rightarrow0.1$ over 200 epochs.
\\[4pt]

RADDT &
tree depth $=5$;
epochs $=300$;
num scale $=5$;
start num $=1$;
warm start: CART (1 initialization).
\\[4pt]

MetaTree &
tree depth $=5$;
max features $=10$;
max samples  $=256$.
\\[4pt]

LLM-Tree &
DeepSeek-V4-Flash (thinking disabled).
\\

\end{longtable}

\subsubsection{Sample Method}
\label{app:sample-method}

\begin{algorithm}[ht]
\caption{Deterministic Valid Train--Evaluation Split}
\label{alg:valid_split}
\begin{algorithmic}[1]
\Require Data set $(X,y)$, sample size $n$, training size $k$,
         seed $s$, minimum class count $m$
\Ensure Ordered split $(X',y')$

\State $n' \gets \min(n, |X|)$
\State Remove missing samples and classes with fewer than two samples

\If{$n < |X|$}
    \State $(\mathcal{C},\mathbf{c}) \gets
        \Call{UniqueWithCounts}{y}$
    \State Initialize $\mathbf{a}\gets\mathbf{0}$
    
    \Comment{Minimum allocation}
    \For{$c_i$ in ascending order of $c_i$}
        \State $a_i\gets
        \min\{m,c_i,n-\sum_j a_j\}$
    \EndFor

    \Comment{Proportional allocation of remaining slots}
    \State $r\gets n-\sum_i a_i$
    \State $q_i\gets
        r\frac{c_i-a_i}{\sum_j(c_j-a_j)}$
    \State $e_i\gets\lfloor q_i\rfloor$
    \State $\ell\gets r-\sum_i e_i$
    \State Let $\mathcal{I}$ be the $\ell$ classes with largest
           $(q_i-e_i)$ among those with $e_i<c_i-a_i$
    \State $e_i\gets e_i+1$ for $i\in\mathcal{I}$
    \State $a_i\gets a_i+e_i$

    \State Sample $a_i$ instances from each class $c_i$
    \State $(X_s,y_s)\gets$ sampled data set
\Else
    \State $(X_s,y_s)\gets(X,y)$
\EndIf

\State $(X_{\mathrm{tr}},X_{\mathrm{te}},
        y_{\mathrm{tr}},y_{\mathrm{te}})
        \gets
        \Call{StratifiedSplit}{X_s,y_s,k,s}$

\State Ensure both splits contain all classes by exchanging samples if necessary

\State \Return
$(X_{\mathrm{tr}},X_{\mathrm{te}},
        y_{\mathrm{tr}},y_{\mathrm{te}})$
\end{algorithmic}
\end{algorithm}

\subsubsection{Preprocessing pipeline and missing value handling}
\label{app:preprocess}

In the original implementation of GradTree, categorical features are processed according to their cardinality. Features with fewer than 10 unique categories are transformed using a \textbf{One-Hot Encoder}, whereas categorical features with higher cardinality are transformed using a \textbf{Leave-One-Out (LOO) Encoder}. Based on our empirical experiments, however, we use the \textbf{Leave-One-Out Encoder for all categorical features}, regardless of their cardinality. This modification provides a unified representation for categorical variables while avoiding the potentially substantial increase in dimensionality introduced by one-hot encoding. Missing values are handled after categorical encoding using median imputation. Specifically, the median of each transformed feature is computed from the training data and used to replace remaining missing values. Following categorical encoding and missing-value imputation, we retain the \textbf{Quantile Transformer} used in the original GradTree implementation and apply it to all transformed features, including both the original numerical features and the LOO-encoded categorical features.

\subsection{Results Acorss All Lenghts}
\label{appendix: overall results across all lengths}

\begin{table}[ht]
\centering
\footnotesize
\caption{Overall mean AUC across benchmarks — model comparison per budget.}
\label{tab:overall_comparison}
\setlength{\tabcolsep}{3pt}  
\begin{tabular}{lllccccc}
\toprule
Benchmark & & Model & 100 & 200 & 400 & 800 & 1000 \\
\midrule
&& CART & 0.668$_{\pm0.121}$ & 0.693$_{\pm0.128}$ & \textbf{0.720}$_{\pm0.131}$ & \underline{0.743}$_{\pm0.130}$ & \underline{0.751}$_{\pm0.126}$ \\
&& CART$_{\leftarrow \mathrm{TabPFN}}$  & 0.655$_{\pm0.114}$ & 0.684$_{\pm0.124}$ & 0.715$_{\pm0.132}$ & \textbf{0.751}$_{\pm0.128}$ & \textbf{0.756}$_{\pm0.126}$ \\
&& $(\alpha^*, \beta^*)$ Tsallis & 0.645$_{\pm0.123}$ & 0.650$_{\pm0.133}$ & 0.655$_{\pm0.141}$ & 0.654$_{\pm0.141}$ & 0.654$_{\pm0.141}$ \\
&& DNDT & 0.629$_{\pm0.089}$ & 0.647$_{\pm0.094}$ & 0.659$_{\pm0.096}$ & 0.665$_{\pm0.094}$ & 0.662$_{\pm0.098}$ \\
&& RADDT & 0.642$_{\pm0.116}$ & 0.651$_{\pm0.121}$ & 0.670$_{\pm0.125}$ & 0.686$_{\pm0.132}$ & 0.688$_{\pm0.135}$ \\
&& DL85 & 0.611$_{\pm0.089}$ & 0.629$_{\pm0.111}$ & 0.659$_{\pm0.124}$ & 0.708$_{\pm0.126}$ & 0.714$_{\pm0.127}$ \\
&& LLM-Tree & 0.630$_{\pm0.090}$ & 0.644$_{\pm0.094}$ & 0.662$_{\pm0.096}$ & 0.673$_{\pm0.096}$ & 0.679$_{\pm0.097}$ \\
&& MetaTree & 0.633$_{\pm0.100}$ & 0.646$_{\pm0.102}$ & 0.647$_{\pm0.106}$ & 0.639$_{\pm0.104}$ & 0.643$_{\pm0.105}$ \\
\cmidrule(lr){2-8} 
\rowcolor{mtgreen!4} 
&&  GradTree & 0.610$_{\pm0.083}$ & 0.653$_{\pm0.101}$ & 0.689$_{\pm0.117}$ & 0.708$_{\pm0.111}$ & 0.712$_{\pm0.116}$ \\
\rowcolor{mtgreen!4} 
&&  GradTree$_{\leftarrow \mathrm{TabPFN}}$ & 0.641$_{\pm0.101}$ & 0.667$_{\pm0.108}$ & 0.701$_{\pm0.119}$ & 0.726$_{\pm0.116}$ & 0.722$_{\pm0.119}$ \\
\rowcolor{mtgreen!4} 
&& MotherTree & \textbf{0.679}$_{\pm0.107}$ & \textbf{0.696}$_{\pm0.107}$ & 0.705$_{\pm0.114}$ & 0.715$_{\pm0.107}$ & 0.716$_{\pm0.110}$  \\
\rowcolor{mtgreen!4} 
 \multirow{-12}{*}{\rotatebox{90}{TabArena}} & \multirow{-4}{*}{\rotatebox{90}{\scriptsize \shortstack{Same tree\\param. space}}} &   MotherTree-Tune & 0.676$_{\pm0.110}$ & 0.696$_{\pm0.119}$ & \underline{0.717}$_{\pm0.124}$ & 0.735$_{\pm0.117}$ & 0.739$_{\pm0.120}$ \\
\midrule
 & & CART & 0.734$_{\pm0.120}$ & \underline{0.777}$_{\pm0.117}$ & \underline{0.805}$_{\pm0.120}$ & 0.811$_{\pm0.123}$ & 0.813$_{\pm0.124}$ \\
&& CART$_{\leftarrow \mathrm{TabPFN}}$ & 0.731$_{\pm0.111}$ & 0.768$_{\pm0.117}$ & 0.802$_{\pm0.117}$ & \underline{0.813}$_{\pm0.117}$ & \underline{0.815}$_{\pm0.118}$ \\
&& $(\alpha^*, \beta^*)$ Tsallis & 0.717$_{\pm0.124}$ & 0.723$_{\pm0.146}$ & 0.724$_{\pm0.163}$ & 0.724$_{\pm0.167}$ & 0.725$_{\pm0.168}$ \\
&& DNDT & 0.585$_{\pm0.066}$ & 0.595$_{\pm0.074}$ & 0.607$_{\pm0.076}$ & 0.616$_{\pm0.082}$ & 0.615$_{\pm0.085}$ \\
&& RADDT & 0.715$_{\pm0.132}$ & 0.748$_{\pm0.136}$ & 0.768$_{\pm0.145}$ & 0.776$_{\pm0.146}$ & 0.777$_{\pm0.149}$ \\
&& DL85 & 0.670$_{\pm0.113}$ & 0.713$_{\pm0.121}$ & 0.747$_{\pm0.136}$ & 0.761$_{\pm0.139}$ & 0.762$_{\pm0.140}$ \\
&& LLM-Tree & 0.684$_{\pm0.135}$ & 0.699$_{\pm0.136}$ & 0.715$_{\pm0.135}$ & 0.718$_{\pm0.137}$ & 0.718$_{\pm0.137}$ \\
&& MetaTree & 0.691$_{\pm0.127}$ & 0.716$_{\pm0.124}$ & 0.726$_{\pm0.123}$ & 0.721$_{\pm0.121}$ & 0.723$_{\pm0.121}$ \\
\cmidrule(lr){2-8} 
\rowcolor{mtgreen!4} 
&&  GradTree & 0.689$_{\pm0.123}$ & 0.734$_{\pm0.124}$ & 0.766$_{\pm0.122}$ & 0.780$_{\pm0.122}$ & 0.786$_{\pm0.122}$ \\
\rowcolor{mtgreen!4} 
& & GradTree$_{\leftarrow \mathrm{TabPFN}}$ & 0.719$_{\pm0.121}$ & 0.762$_{\pm0.117}$ & 0.784$_{\pm0.121}$ & 0.795$_{\pm0.119}$ & 0.799$_{\pm0.119}$ \\
\rowcolor{mtgreen!4}
\multirow{-12}{*}{\rotatebox{90}{OpenML-CC18\scriptsize -Small}}&& MotherTree &  \underline{0.741}$_{\pm0.111}$ & 0.765$_{\pm0.111}$ & 0.780$_{\pm0.114}$ & 0.784$_{\pm0.114}$ & 0.784$_{\pm0.114}$ \\
\rowcolor{mtgreen!4} 
&\multirow{-4}{*}{\rotatebox{90}{\scriptsize \shortstack{Same tree\\param. space}}} & MotherTree-Tune & \textbf{0.766}$_{\pm0.117}$ & \textbf{0.786}$_{\pm0.121}$ & \textbf{0.811}$_{\pm0.113}$ & \textbf{0.815}$_{\pm0.116}$ & \textbf{0.816}$_{\pm0.117}$ \\
\bottomrule
\label{tab:overall}
\end{tabular}
\end{table}

\begin{table}[ht]
\centering
\footnotesize
\caption{Mean AUC for binary classification tasks on TabArena and OpenML-CC18}
\label{tab:binary_tasks}
\setlength{\tabcolsep}{3pt}  

\begin{tabular}{lllccccc}
\toprule
Benchmark & & Model & 100 &  200 & 400 & 800 & 1000 \\
\midrule
& & CART & 0.647$_{\pm0.101}$ & 0.670$_{\pm0.107}$ & 0.697$_{\pm0.106}$ & 0.722$_{\pm0.108}$ & \underline{0.731} $_{\pm0.104}$ \\
& & CART$_{\leftarrow \mathrm{TabPFN}}$ & 0.634$_{\pm0.092}$ & 0.659$_{\pm0.099}$ & 0.690$_{\pm0.108}$ & \textbf{0.732}$_{\pm0.107}$ & \textbf{0.738}$_{\pm0.104}$ \\
& & $(\alpha^*, \beta^*)$ Tsallis & 0.619$_{\pm0.100}$ & 0.620$_{\pm0.108}$ & 0.624$_{\pm0.118}$ & 0.623$_{\pm0.116}$ & 0.623$_{\pm0.117}$ \\
& & DNDT & 0.586$_{\pm0.071}$ & 0.595$_{\pm0.079}$ & 0.607$_{\pm0.079}$ & 0.614$_{\pm0.088}$ & 0.616$_{\pm0.089}$ \\
& & RADDT & 0.620$_{\pm0.095}$ & 0.627$_{\pm0.100}$ & 0.646$_{\pm0.108}$ & 0.660$_{\pm0.115}$ & 0.662$_{\pm0.114}$ \\
 & & DL85 & 0.594$_{\pm0.073}$ & 0.604$_{\pm0.088}$ & 0.631$_{\pm0.096}$ & 0.685$_{\pm0.103}$ & 0.692$_{\pm0.106}$ \\
& &  MetaTree & 0.614$_{\pm0.076}$ & 0.632$_{\pm0.085}$ & 0.640$_{\pm0.090}$ & 0.631$_{\pm0.087}$ & 0.633$_{\pm0.092}$ \\
& & LLM-Tree & 0.629$_{\pm0.091}$ & 0.641$_{\pm0.092}$ & 0.661$_{\pm0.094}$ & 0.671$_{\pm0.095}$ & 0.679$_{\pm0.096}$ \\
\rowcolor{mtgreen!4} 
& &  GradTree & 0.604$_{\pm0.082}$ & 0.635$_{\pm0.081}$ & 0.674$_{\pm0.101}$ & 0.695$_{\pm0.097}$ & 0.700$_{\pm0.102}$ \\
\rowcolor{mtgreen!4} 
& &  GradTree$_{\leftarrow \mathrm{TabPFN}}$ & 0.629$_{\pm0.093}$ & 0.655$_{\pm0.095}$ & 0.684$_{\pm0.103}$ & 0.715$_{\pm0.102}$ & 0.711$_{\pm0.104}$ \\
\rowcolor{mtgreen!4} 
& & MotherTree & \textbf{0.668}$_{\pm0.105}$ & \textbf{0.692}$_{\pm0.104}$ & \textbf{0.712}$_{\pm0.106}$ & \underline{0.724}$_{\pm0.103}$ & 0.722$_{\pm0.100}$ \\
\rowcolor{mtgreen!4} 
\multirow{-12}{*}{\rotatebox{90}{TabArena}} & \multirow{-4}{*}{\rotatebox{90}{\scriptsize \shortstack{Same tree\\param. space}}}  & MotherTree-Tune & \underline{0.663}$_{\pm0.092}$ & \underline{0.680}$_{\pm0.101}$ & \underline{0.701}$_{\pm0.105}$ & 0.723$_{\pm0.098}$ & 0.728 $_{\pm0.104}$ \\
\cline{1-7}
 & &  CART & 0.709$_{\pm0.129}$ & 0.759$_{\pm0.123}$ & \underline{0.780}$_{\pm0.122}$ & 0.783$_{\pm0.123}$ & 0.783$_{\pm0.123}$ \\
& & CART$_{\leftarrow \mathrm{TabPFN}}$ & 0.712$_{\pm0.120}$ & 0.748$_{\pm0.125}$ & 0.777$_{\pm0.120}$ & \underline{0.788}$_{\pm0.115}$ & \underline{0.788}$_{\pm0.115}$ \\
& & $(\alpha^*, \beta^*)$ Tsallis & 0.684$_{\pm0.133}$ & 0.680$_{\pm0.155}$ & 0.675$_{\pm0.164}$ & 0.672$_{\pm0.161}$ & 0.672$_{\pm0.161}$ \\
& & DNDT & 0.635$_{\pm0.106}$ & 0.650$_{\pm0.108}$ & 0.658$_{\pm0.109}$ & 0.660$_{\pm0.107}$ & 0.660$_{\pm0.109}$ \\
& & RADDT & 0.684$_{\pm0.136}$ & 0.710$_{\pm0.135}$ & 0.726$_{\pm0.144}$ & 0.732$_{\pm0.143}$ & 0.730$_{\pm0.143}$ \\
& & DL85 & 0.657$_{\pm0.128}$ & 0.691$_{\pm0.131}$ & 0.711$_{\pm0.142}$ & 0.721$_{\pm0.141}$ & 0.721$_{\pm0.141}$ \\
& & MetaTree & 0.687$_{\pm0.145}$ & 0.705$_{\pm0.141}$ & 0.711$_{\pm0.141}$ & 0.711$_{\pm0.141}$ & 0.711$_{\pm0.141}$ \\
& & LLM-Tree & 0.659$_{\pm0.147}$ & 0.669$_{\pm0.145}$ & 0.686$_{\pm0.142}$ & 0.687$_{\pm0.143}$ & 0.687$_{\pm0.143}$ \\
\rowcolor{mtgreen!4} 
& & GradTree & 0.705$_{\pm0.136}$ & 0.725$_{\pm0.132}$ & 0.752$_{\pm0.127}$ & 0.765$_{\pm0.128}$ & 0.766$_{\pm0.125}$ \\
\rowcolor{mtgreen!4} 
& & GradTree$_{\leftarrow \mathrm{TabPFN}}$ & 0.708$_{\pm0.137}$ & 0.752$_{\pm0.128}$ & 0.770$_{\pm0.135}$ & 0.780$_{\pm0.127}$ & 0.785$_{\pm0.127}$ \\
\rowcolor{mtgreen!4} 
& & MotherTree & \underline{0.747}$_{\pm0.124}$ & \underline{0.766}$_{\pm0.120}$ & 0.775$_{\pm0.124}$ & 0.776$_{\pm0.124}$ & 0.776$_{\pm0.124}$ \\
\rowcolor{mtgreen!4} 
\multirow{-12}{*}{\rotatebox{90}{OpenML-CC18\scriptsize -Small}} & \multirow{-4}{*}{\rotatebox{90}{\scriptsize \shortstack{Same tree\\param. space}}} & MotherTree-Tune & \textbf{0.759}$_{\pm0.124}$ & \textbf{0.772}$_{\pm0.125}$ & \textbf{0.796}$_{\pm0.117}$ & \textbf{0.797}$_{\pm0.120}$ & \textbf{ 0.797}$_{\pm0.120}$ \\

\bottomrule
\end{tabular}
\end{table}

\begin{table}[ht]
\centering
\footnotesize
\caption{ Mean AUC for multi-class tasks on TabArena and OpenML-CC18.}
\label{tab:multi-class task}
\setlength{\tabcolsep}{3pt}  

\begin{tabular}{lllccccc}
\toprule
Benchmark & & Model & 100 &  200 & 400 & 800 & 1000 \\
\midrule
 & & CART & \textbf{0.748}$_{\pm0.161}$ & \underline{0.779}$_{\pm0.170}$ & \textbf{0.808}$_{\pm0.180}$ & \underline{0.819}$_{\pm0.178}$ & \underline{0.824}$_{\pm0.176}$ \\
&  &CART$_{\leftarrow \mathrm{TabPFN}}$ & 0.736$_{\pm0.155}$ & \textbf{0.780}$_{\pm0.165}$ & \underline{0.806}$_{\pm0.177}$ & \textbf{0.822}$_{\pm0.179}$ & \textbf{0.826}$_{\pm0.180}$ \\
& & $(\alpha^*, \beta^*)$ Tsallis & \underline{0.744}$_{\pm0.156}$ & 0.762$_{\pm0.167}$ & 0.769$_{\pm0.168}$ & 0.772$_{\pm0.171}$ & 0.771$_{\pm0.170}$ \\
& & DNDT & 0.578$_{\pm0.044}$ & 0.595$_{\pm0.055}$ & 0.605$_{\pm0.068}$ & 0.622$_{\pm0.057}$ & 0.608$_{\pm0.073}$ \\
& & RADDT & 0.724$_{\pm0.153}$ & 0.744$_{\pm0.152}$ & 0.760$_{\pm0.149}$ & 0.780$_{\pm0.159}$ & 0.785$_{\pm0.171}$ \\
& & DL85 & 0.675$_{\pm0.117}$ & 0.722$_{\pm0.142}$ & 0.763$_{\pm0.164}$ & 0.792$_{\pm0.174}$ & 0.798$_{\pm0.170}$ \\
& & MetaTree & 0.703$_{\pm0.147}$ & 0.696$_{\pm0.145}$ & 0.673$_{\pm0.159}$ & 0.670$_{\pm0.155}$ & 0.681$_{\pm0.145}$ \\
& & LLM-Tree & 0.634$_{\pm0.096}$ & 0.654$_{\pm0.105}$ & 0.669$_{\pm0.109}$ & 0.678$_{\pm0.108}$ & 0.679$_{\pm0.107}$ \\
\rowcolor{mtgreen!4} 
& & GradTree & 0.631$_{\pm0.091}$ & 0.723$_{\pm0.141}$ & 0.748$_{\pm0.157}$ & 0.756$_{\pm0.153}$ & 0.757$_{\pm0.159}$ \\
\rowcolor{mtgreen!4} 
& & GradTree$_{\leftarrow \mathrm{TabPFN}}$ & 0.686$_{\pm0.124}$ & 0.715$_{\pm0.146}$ & 0.761$_{\pm0.161}$ & 0.767$_{\pm0.160}$ & 0.763$_{\pm0.165}$ \\
\rowcolor{mtgreen!4} 
& & MotherTree & 0.675$_{\pm0.119}$ & 0.698$_{\pm0.133}$ & 0.716$_{\pm0.136}$ & 0.726$_{\pm0.137}$ & 0.727$_{\pm0.134}$ \\
\rowcolor{mtgreen!4} 
\multirow{-12}{*}{\rotatebox{90}{TabArena}} & \multirow{-4}{*}{\rotatebox{90}{\scriptsize \shortstack{Same tree\\param. space}}} & MotherTree-Tune & 0.728$_{\pm0.159}$ & 0.757$_{\pm0.166}$ & 0.779$_{\pm0.172}$ & 0.778$_{\pm0.174}$ & 0.779$_{\pm0.170}$ \\
\cline{1-7}
 & & CART & \underline{0.767}$_{\pm0.104}$ & \underline{0.802}$_{\pm0.110}$ & \textbf{0.836}$_{\pm0.115}$ & \textbf{0.849}$_{\pm0.117}$ & \textbf{0.853}$_{\pm0.119}$ \\
&  & CART$_{\leftarrow \mathrm{TabPFN}}$ & 0.756$_{\pm0.098}$ & 0.795$_{\pm0.105}$ & \underline{0.835}$_{\pm0.110}$ & \underline{0.846}$_{\pm0.115}$ & \underline{0.849}$_{\pm0.116}$ \\
&  & $(\alpha^*, \beta^*)$ Tsallis & 0.760$_{\pm0.101}$ & 0.780$_{\pm0.115}$ & 0.789$_{\pm0.143}$ & 0.792$_{\pm0.156}$ & 0.793$_{\pm0.158}$ \\
&  & DNDT & 0.620$_{\pm0.065}$ & 0.643$_{\pm0.076}$ & 0.660$_{\pm0.081}$ & 0.671$_{\pm0.080}$ & 0.666$_{\pm0.087}$ \\
&  & RADDT & 0.755$_{\pm0.119}$ & 0.798$_{\pm0.124}$ & 0.822$_{\pm0.133}$ & 0.833$_{\pm0.136}$ & 0.837$_{\pm0.139}$ \\
& & DL85 & 0.686$_{\pm0.093}$ & 0.743$_{\pm0.104}$ & 0.793$_{\pm0.115}$ & 0.814$_{\pm0.123}$ & 0.815$_{\pm0.124}$ \\
& & MetaTree & 0.697$_{\pm0.104}$ & 0.731$_{\pm0.103}$ & 0.744$_{\pm0.097}$ & 0.734$_{\pm0.091}$ & 0.737$_{\pm0.092}$ \\
& & LLM-Tree & 0.718$_{\pm0.115}$ & 0.739$_{\pm0.118}$ & 0.754$_{\pm0.121}$ & 0.758$_{\pm0.123}$ & 0.759$_{\pm0.122}$ \\
\rowcolor{mtgreen!4}
& & GradTree & 0.669$_{\pm0.105}$ & 0.746$_{\pm0.118}$ & 0.786$_{\pm0.117}$ & 0.800$_{\pm0.115}$ & 0.812$_{\pm0.117}$ \\
\rowcolor{mtgreen!4}
& & GradTree$_{\leftarrow \mathrm{TabPFN}}$ & 0.732$_{\pm0.098}$ & 0.774$_{\pm0.104}$ & 0.803$_{\pm0.103}$ & 0.816$_{\pm0.111}$ & 0.818$_{\pm0.111}$ \\
\rowcolor{mtgreen!4}
&  & MotherTree & 0.733$_{\pm0.097}$ & 0.764$_{\pm0.103}$ & 0.787$_{\pm0.103}$ & 0.795$_{\pm0.103}$ & 0.795$_{\pm0.103}$ \\
\rowcolor{mtgreen!4}
\multirow{-12}{*}{\rotatebox{90}{OpenML-CC18\scriptsize -Small}} & \multirow{-4}{*}{\rotatebox{90}{\scriptsize \shortstack{Same tree\\param. space}}}  & MotherTree-Tune & \textbf{0.776}$_{\pm0.111}$ & \textbf{0.804}$_{\pm0.117}$ & 0.831$_{\pm0.108}$ & 0.838$_{\pm0.111}$ & 0.840$_{\pm0.112}$ \\

\bottomrule
\end{tabular}
\end{table}

\newpage
\subsection{Running Speed}


\begin{table}[ht]
\centering
\footnotesize
\caption{Average runtime (tree fitting and test-set inference) comparison on GPU and CPU for OpenML-CC18-Small over varying training set size.}
\label{tab:runtime}
\setlength{\tabcolsep}{5pt}
\begin{tabular}{@{} llccccc @{}}
\toprule
& & \multicolumn{5}{c}{Training set budget, $B$} \\
Hardware & Model & 100 & 200 & 400 & 800 & 1000 \\
\midrule
\multirow{3}{*}{GPU}
& CART$_{\leftarrow \mathrm{TabPFN}}$
& 0.72 & 0.66 & 0.82 & 0.96 & 1.02 \\
& MotherTree
& 0.15 & 0.15 & 0.18 & 0.19 & 0.20 \\
& MotherTree-Tune
& 1.57 & 2.15 & 3.15 & 3.82 & 4.75 \\
\midrule
\multirow{2}{*}{CPU}
& CART$_{\leftarrow \mathrm{TabPFN}}$
& 1.07 & 1.64 & 2.85 & 4.33 & 4.90 \\
& MotherTree
& 0.21 & 0.17 & 0.24 & 0.33 & 0.35 \\
& MotherTree-Tune
& 2.25 & 2.85 & 4.03 & 4.96 & 5.47 \\
\bottomrule
\end{tabular}
\end{table}

Table~\ref{tab:runtime} compares the average runtime (tree fitting and test set inference) of MotherTree and MotherTree-Tune with the TabPFN-distilled CART baseline on
OpenML-CC18-Small. Direct MotherTree inference is consistently the
fastest method on both GPU and CPU, requiring only about
$0.15$--$0.20$ seconds on GPU and $0.17$--$0.35$ seconds on CPU across
the evaluated training-set sizes. Its runtime also grows only mildly
with the number of training samples, reflecting the fact that a
task-specific tree is generated in a single forward pass.

MotherTree-Tune incurs a substantially higher cost because it performs
additional task-specific optimization after tree generation. Its
runtime increases with the training-set size, reaching $4.75$ seconds
on GPU and $5.47$ seconds on CPU at $B=1000$. This additional cost
trades computational efficiency for improved predictive performance,
while remaining within a few seconds per task in the evaluated
settings.

\newpage

\subsection{Complete Critical Difference Diagram}
\label{app:cd}
Figure~\ref{fig:cdds-all-joint} contains critical difference diagrams for all models at all training set budgets, $B$.
\begin{figure}[ht]
    \centering

    \begin{subfigure}{0.49\textwidth}
        \centering
        \includegraphics[width=\textwidth]{figure/join/100/cd_diagram.pdf}
        \caption{$B  = 100$}
        \label{fig:exp1}
    \end{subfigure}
    \hfill
    \begin{subfigure}{0.49\textwidth}
        \centering
        \includegraphics[width=\textwidth]{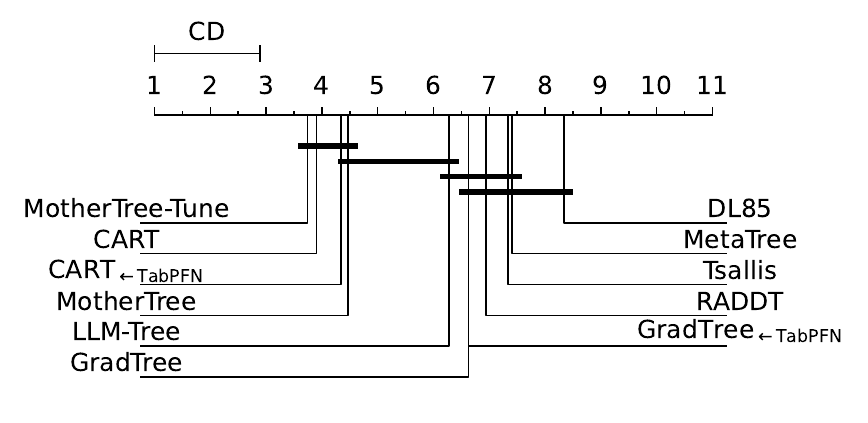}
        \caption{$B  = 200$}
        \label{fig:exp2}
    \end{subfigure}
    \hfill
    \begin{subfigure}{0.49\textwidth}
        \centering
        \includegraphics[width=\textwidth]{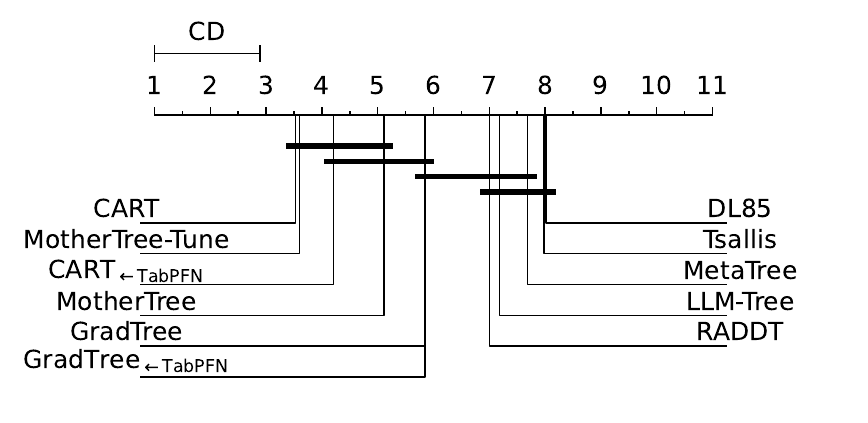}
        \caption{$B = 400$}
        \label{fig:exp3_400}
    \end{subfigure}
    \begin{subfigure}{0.49\textwidth}
        \centering
        \includegraphics[width=\textwidth]{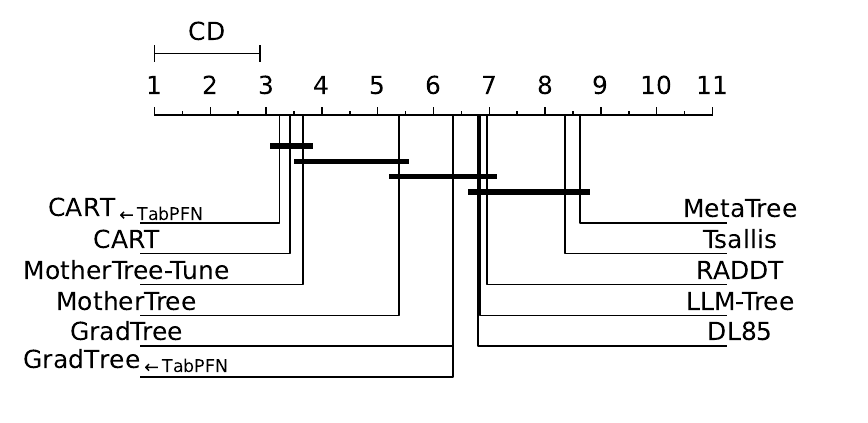}
        \caption{$B = 800 $}
        \label{fig:exp3_800}
    \end{subfigure}

    \begin{subfigure}{0.49\textwidth}
        \centering
        \includegraphics[width=\textwidth]{figure/join/1000/cd_diagram.pdf}
        \caption{$B = 1000$}
        \label{fig:exp3_1000}
    \end{subfigure}
    \caption{Critical difference diagrams with different maximum context lengths for two joint benchmarks}
    \label{fig:cdds-all-joint}
\end{figure}

\newpage

\subsection{XOR Benchmark}
\label{app:xor-benchmark}

\subsubsection{Synthetic XOR data sets with distractor features}
\label{XOR subsection}

\begin{algorithm}[ht]
\caption{XOR Data set Generation Process}
\label{alg:parity_data set}
\begin{algorithmic}[1]

\Require Number of samples $N$, parity bits $k$, 
number of distractors $m$, agreement probability $q$, random seed $s$
\Ensure Data set $(X,y)$ and feature names $\mathcal{F}$

\State Initialize random number generator with seed $s$

\State Sample $B \in \{0,1\}^{N \times k}$ with
       $B_{ij} \sim \mathrm{Bernoulli}(0.5)$

\State Compute parity labels
       $y_i \gets \left(\sum_{j=1}^{k} B_{ij}\right) \bmod 2$
       for $i=1,\ldots,N$

\State $X \gets B$

\State $\mathcal{F} \gets
       \{\texttt{bit\_0},\ldots,\texttt{bit\_}\{k-1\}\}$

\For{$d = 0$ to $m-1$}

    \State Sample noise vector
           $\epsilon^{(d)} \in \{0,1\}^{N}$,
           where $\epsilon^{(d)}_i
           \sim \mathrm{Bernoulli}(1-q)$

    \State Construct distractor
           $z^{(d)} \gets \bigl(y + \epsilon^{(d)}\bigr) \bmod 2$

    \State $X \gets [X \mid z^{(d)}]$

    \State Append \texttt{distract\_$d$} to $\mathcal{F}$

\EndFor

\State \Return $(X,y,\mathcal{F})$

\end{algorithmic}
\end{algorithm}

We set  $k \in \{2, 3, 4, 5\}$ parity bits, which defines the difficulty of addressing the parity problems.  Number of distractors $m$ and $\lambda$ define the quantity of sub-optimal features that trees are allowed to select. 
$q \in \{0.50, 0.55, 0.60, 0.65, 0.70 \}$  to be equals to $y$  and is flipped otherwise. Probability $q$ defines the correlation between distractors and desirable features. High $q$ will mislead the model to achieve high results even if distractors themselves are not optimal.  From the above-mentioned procedure, it will yield $d = k + m \in \{4, 5, 6, 7\}$ binary features and we generate $n = 2000$ samples per data set. The detailed data set generation process is illustrated in Algorithm \ref{alg:parity_data set} and corresponding configuration is listed in Table \ref{tab:xor}.  Appendix \ref{app:distractor_auc} provides the closed form value for AUC ceiling using only distractors. 

In the experiments, we repeat running each $(k, m)$ pair for 10 times with different seeds. We select DL85 as the optimal tree solver to gauge the gap with other tree models. GradTree and CART together with their distilled versions are used as baselines to compare with both MotherTree and MotherTree-Tune.

\begin{table}[ht]
\centering
\caption{XOR/parity data set configurations.  The tree depth is fixed at
$D = 5$ throughout.  $k$ is the number of parity bits ($y = \mathrm{XOR}$ of the
$k$ bits), $m$ the number of distractor features (independent noisy copies of
$y$, each agreeing with it with probability $q$), $d = k + m$ the feature
dimensionality (all features binary) and $\lambda = m - (D - k)$ is the number of
distractors a model must leave \emph{unused} for the parity tree to still fit
below them. $\lambda \leq 0$ indicates model can incorporate all bits even if all distractors are captured.  Every $(k, m)$ pair is generated for
$q \in \{0.50, 0.55, 0.60, 0.65, 0.70\}$; the entries of the $m$, $\lambda$ and
$d$ columns correspond positionally.  One seed per configuration gives $59$
data sets in total.}
\label{tab:xor}
\begin{tabular}{cccccr}
\toprule
$k$ (bits) & $m$ (distractors) & $\lambda$ & $q$ & $d$ & \#data sets \\
\midrule
2 & 2, 3, 4 & $-1,\;0,\;+1$ & \multirow{4}{*}{$0.50$--$0.70$} & 4, 5, 6 & 15 \\
3 & 1, 2, 3 & $-1,\;0,\;+1$ & & 4, 5, 6 & 15 \\
4 & 1, 2, 3 & $0,\;+1,\;+2$ & & 5, 6, 7 & 15 \\
5 & 1, 2    & $+1,\;+2$    & & 6, 7    & 10 \\
\midrule
2, 3, 4, 5 & 0 (parity control) & $-3,\;-2,\;-1,\;0$ & $0.55$ & $k$ & 4 \\
\midrule
\multicolumn{5}{r}{\textbf{total}} & \textbf{59} \\
\bottomrule
\end{tabular}
\end{table}

\subsubsection{Distractor-Only Bayes AUC Ceiling}
\label{app:distractor_auc}

We derive the best population AUC achievable using only the distractor
features. This provides a distractor-only ceiling for the synthetic parity
experiments.

Let $Y\in\{0,1\}$ be balanced and let
$Z=(Z_1,\ldots,Z_m)\in\{0,1\}^m$ denote the distractors. Each distractor is
an independent noisy copy of the label:
\[
\Pr(Z_j=Y\mid Y)=q,\qquad q>\frac12,
\]
and the $Z_j$ are conditionally independent given $Y$.

Define the distractor-count score
\[
S(Z)=\sum_{j=1}^m Z_j.
\]
Let
\[
Z^+ \sim P(Z\mid Y=1),
\qquad
Z^- \sim P(Z\mid Y=0)
\]
denote the distractor vectors of an independently drawn positive and
negative example, respectively. We further define
\[
S_+ := S(Z^+),
\qquad
S_- := S(Z^-).
\]
Thus, $S_+$ is the number of distractors equal to one for a positive
example, while $S_-$ is the corresponding number for a negative example.

\begin{thmprop}
\label{prop:distractor_auc}
The Bayes-optimal distractor-only score $r(z)=\Pr(Y=1\mid Z=z)$ is a strictly increasing function of $S(z)$. Therefore, $\mathrm{AUC}_{\mathrm{dist}}^*(m,q)
=
\Pr(S_+>S_-)
+
\frac12\Pr(S_+=S_-)$, where

$S_+\sim\mathrm{Binomial}(m,q), S_-\sim\mathrm{Binomial}(m,1-q).$
\end{thmprop}

\begin{proof}
Consider a realization $z=(z_1,\ldots,z_m)$ containing
$s=\sum_j z_j$ ones. Conditional independence gives
\[
\Pr(Z=z\mid Y=1)
=
q^s(1-q)^{m-s},
\]
and
\[
\Pr(Z=z\mid Y=0)
=
(1-q)^s q^{m-s}.
\]
Since the classes are balanced, Bayes' rule yields
\[
r(z)
=
\Pr(Y=1\mid Z=z)
=
\frac{
q^s(1-q)^{m-s}
}{
q^s(1-q)^{m-s}
+
(1-q)^s q^{m-s}
}.
\]
Hence $r(z)$ depends on $z$ only through $s=S(z)$, i.e.,
\[
\Pr(Y=1\mid Z=z)
=
\Pr(Y=1\mid S=s).
\]
Equivalently,
\[
r(s)
=
\frac{1}{
1+\left(\frac{1-q}{q}\right)^{2s-m}
}.
\]
Because $q>1/2$, $r(s)$ is strictly increasing in $s$.

AUC is the probability that a randomly drawn positive example receives a
larger score than a randomly drawn negative example, with ties receiving
half credit:
\[
\mathrm{AUC}
=
\Pr(r(Z^+)>r(Z^-))
+
\frac12\Pr(r(Z^+)=r(Z^-)),
\]
where $Z^+\sim P(Z\mid Y=1)$ and
$Z^-\sim P(Z\mid Y=0)$. Since $r$ is strictly increasing in $S$, this is
equivalent to
\[
\mathrm{AUC}
=
\Pr(S_+>S_-)
+
\frac12\Pr(S_+=S_-).
\]
Finally, for a positive example each distractor equals $1$ with probability
$q$, whereas for a negative example it equals $1$ with probability $1-q$.

Therefore,
\[
S_+\sim\mathrm{Binomial}(m,q),
\qquad
S_-\sim\mathrm{Binomial}(m,1-q),
\]
which proves the result.
\end{proof}

For the values of $m$ used in our experiments, the resulting closed forms
are
\[
\mathrm{AUC}_{\mathrm{dist}}^*(1,q)=q, \qquad \mathrm{AUC}_{\mathrm{dist}}^*(2,q)=3q^2-2q^3,
\]
\[
\mathrm{AUC}_{\mathrm{dist}}^*(3,q)
=
10q^3-15q^4+6q^5, \qquad 
\mathrm{AUC}_{\mathrm{dist}}^*(4,q)
=
35q^4-84q^5+70q^6-20q^7.
\]

Since all $k$ parity bits deterministically determine the label,
\[
Y=X_1\oplus\cdots\oplus X_k,
\]
the full-feature optimum is $\mathrm{AUC}^*_{\mathrm{full}}=1$ whenever the
model class can represent the depth-$k$ parity tree. Therefore, population
AUC significantly above
$\mathrm{AUC}_{\mathrm{dist}}^*(m,q)$ cannot be explained by the
distractors alone and provides evidence that the predictor exploits
information contained in the parity features.

\subsubsection{Figures in all settings for Parity(XOR) tasks}
\label{app:xor_full_fig}

\begin{figure}[ht]
    \centering
    \vfill
    \begin{minipage}[t]{\linewidth}
        \centering
        \includegraphics[width=\linewidth]
        {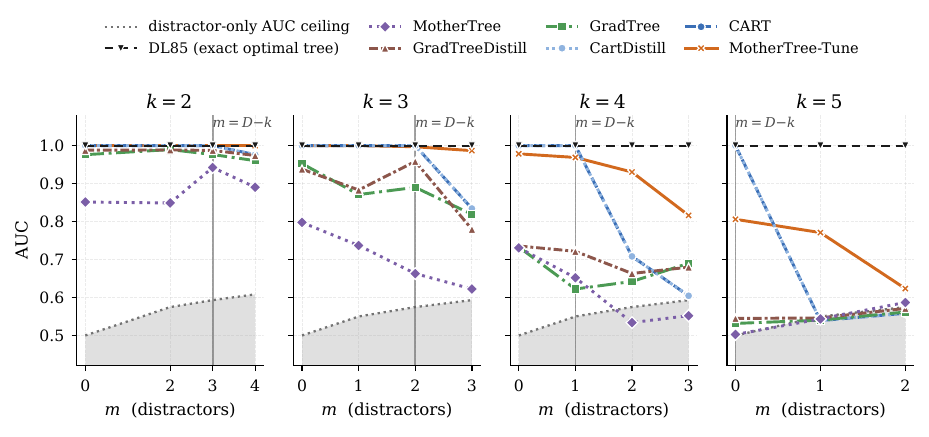}
        \subcaption{AUC against the number of distractor features $m$. }
        \label{fig:xor_panels_app}
    \end{minipage}
    \vfill
    \begin{minipage}[t]{\linewidth}
        \centering
        \includegraphics[width=\linewidth]
        {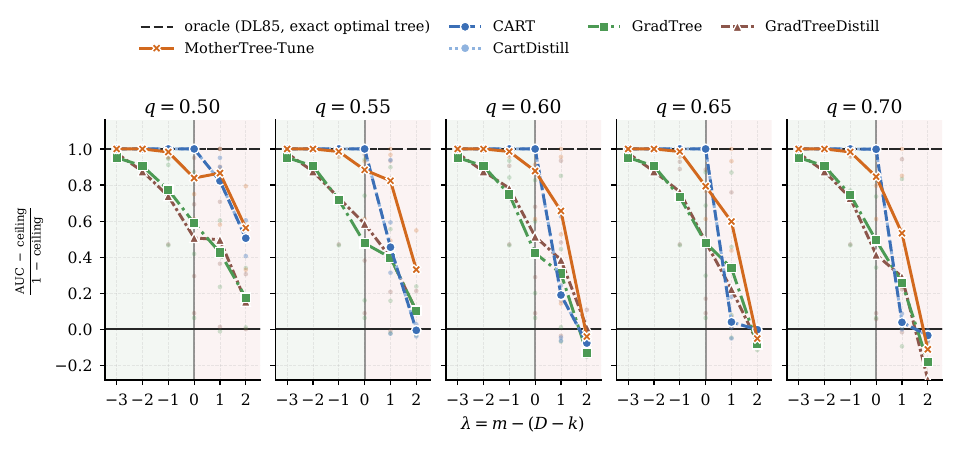}
        \subcaption{Share of AUC under different $q$(label correlations)}. 
        \label{fig:xor_q_facet_app}
    \end{minipage}
    \caption{%
        Parity with distractors at a fixed depth budget $D = 5$.
        $\lambda = m - (D-k)$ is the number of distractors that must be left to learn the parity tree. Green ($\lambda \le 0$): the depth budget has slack.  Splitting on a distractor does not destroy the label. 
Red ($\lambda > 0$): the budget is exhausted, so each distractor used takes a
level away from the parity tree; at least $\lambda$ of them must be refused.}
\label{fig:xor_app}
\end{figure}

\end{document}

%% file: math_commands.tex
\usepackage{amsmath,amsfonts,bm}

\newtheorem{thmprop}{Proposition}

\newtheorem*{thmidasmp*}{Identifying assumptions}

\theoremstyle{definition}

\newtheorem*{thmrem*}{Remark}
\newtheorem*{thmprop*}{Proposition}

\def\eqref#1{equation~\ref{#1}}

\def\1{\bm{1}}

\DeclareMathAlphabet{\mathsfit}{\encodingdefault}{\sfdefault}{m}{sl}
\SetMathAlphabet{\mathsfit}{bold}{\encodingdefault}{\sfdefault}{bx}{n}

\def\gP{{\mathcal{P}}}

